\documentclass[twoside]{article}
\usepackage[accepted]{aistats2018}

\newenvironment{packeditemize}{\begin{list}{$\bullet$}{\setlength{\itemsep}{0pt}\addtolength{\labelwidth}{-5pt}\setlength{\leftmargin}{\labelwidth}\setlength{\listparindent}{\parindent}\setlength{\parsep}{0pt}\setlength{\topsep}{3pt}}}{\end{list}}

\usepackage{amsmath,graphicx,bm}
\usepackage[ruled,noend]{algorithm2e}
\usepackage{color}
\usepackage{amsthm}
\usepackage{amsfonts}       
\usepackage{booktabs}       
\usepackage[titletoc,title]{appendix}
\usepackage{hyperref}       
\usepackage{url,cite}            
\usepackage[font=scriptsize]{caption}
\usepackage{mwe}

\newtheorem{theorem}{Theorem}

\newtheorem{lemma}{Lemma}
\newtheorem{sublemma}{Lemma}[lemma]

\newcommand{\assign}{:=}
\newcommand{\nobracket}{}
\newcommand{\nocomma}{}
\newcommand{\tmop}[1]{\ensuremath{\operatorname{#1}}}

\newcommand\numberthis{\addtocounter{equation}{1}\tag{\theequation}}
\newcommand{\noplus}{}

\begin{document}
%

%

\runningauthor{Wu, Ioannidis, Sznaier, Li, Kaeli, Dy}
\twocolumn[
\aistatstitle{Iterative Spectral Method for Alternative Clustering}
\aistatsauthor{Chieh Wu \And Stratis Ioannidis \And Mario Sznaier}
\aistatsauthor{Xiangyu Li \And David Kaeli \And Jennifer G. Dy}
\aistatsaddress{Electrical and Computer Engineering Dept., Northeastern University, Boston, MA}
]

%
%
%
%
\begin{abstract}
Given a dataset and an existing clustering as input, alternative clustering aims to find an alternative partition.  One of the state-of-the-art approaches is Kernel Dimension Alternative Clustering (KDAC). We propose a novel Iterative Spectral Method (ISM)   that greatly improves the scalability of KDAC. Our algorithm is intuitive, relies on easily implementable spectral decompositions, and comes with theoretical guarantees.   Its computation time improves upon existing implementations of KDAC by as much as 5 orders of magnitude.
\end{abstract}

\section{Introduction}

Clustering, i.e., the process of grouping similar objects in a dataset together, is a classic problem. It is extensively used for exploratory data analysis.  Traditional clustering algorithms typically identify a single partitioning of a given dataset.  However, data is often multi-faceted and can be both interpreted and clustered through multiple viewpoints (or, {\em views}).  For example, the same face data can be clustered based on either identity or based on pose.  In real applications, partitions generated by a clustering algorithm may not correspond to the view a user is interested in.  

In this paper, we address the problem of finding an {\em alternative clustering}, given a dataset and an existing, pre-computed clustering.  Ideally, one would like the alternative clustering to be {\em novel} (i.e., non-redundant) w.r.t. the existing clustering to reveal a new viewpoint to the user.  Simultaneously, one would like the result to reveal partitions of high clustering {\em quality}.
 Several recent papers propose algorithms for  alternative clustering~\cite{gondek2007non,cui2010learning,dang2010generation,davidson2008finding,cui2007non,niu2014iterative}.  Among them, Kernel Dimension Alternative Clustering (KDAC) is a flexible approach, shown to have superior performance compared to several competitors~\cite{niu2014iterative}. KDAC is as powerful as spectral clustering in discovering arbitrarily-shaped clusters (including ones that are not linearly separable) that are non-redundant w.r.t.~an existing clustering.  As an additional advantage, KDAC can simultaneously learn the subspace in which the alternative clustering resides.  

The flexibility of KDAC comes at a price: the KDAC formulation involves optimizing a non-convex cost function constrained over the space of orthogonal matrices (i.e, the Stiefel manifold). Niu et al.~\cite{niu2014iterative} proposed a Dimension Growth (DG) heuristic for solving this optimization problem, which is nevertheless highly computationally intensive. We elaborate on its complexity in Section~\ref{gen_inst}; experimentally, DG is quite slow, with a convergence time of roughly $46$ hours on an Intel Xeon CPU, for a $624$ sample-sized face data (c.f.~Section~\ref{sec:exp}).  This limits the applicability of KDAC in interactive exploratory data analysis settings, which often require results to be presented to a user within a few seconds.  It also limits the scalability of KDAC to large data.  Alternately, one can solve the KDAC optimization problem by gradient descent on a Stiefel manifold (SM)~\cite{wen2013feasible}. However, given the lack of convexity, both DG or SM are prone to get trapped to local minima. Multiple iterations  with random initializations are required to ameliorate the effect of locality.  This increases computation time, and also decreases in effectiveness as the dimensionality of the data increases: the increase in dimension rapidly expands the search space and the abundance of local minima. As such, with both DG and SM, the clustering quality is negatively affected by an increase in dimension.

{\bf Our Contributions.}   
Motivated by the above issues, we make the following contributions:
\begin{packeditemize}
\item We propose an Iterative Spectral Method (ISM), a \emph{novel algorithm} for solving the non-convex optimization constrained on a Stiefel manifold problem inherent in KDAC. Our algorithm has  several highly desirable properties. First, it \emph{significantly outperforms} traditional methods such as DG and SM in terms of both computation time and quality of the produced alternative clustering.  Second, the algorithm relies on an \emph{intuitive use of iterative spectral decompositions}, making it  both easy  to  understand as well as easy to implement, using off-the-shelf libraries. 
\item  ISM has a natural initialization, constructed through a Taylor approximation of the problem's Lagrangian. Therefore, high quality results can be obtained without random restarts in search of a better initialization. We show that this initialization is a contribution in its own right, as its use improves performance of competitor algorithms.
\item We provide \emph{theoretical guarantees} on its fixed point.  In particular, we establish conditions under which  the fixed point of ISM satisfies both the 1st and 2nd order necessary conditions for local optimality. 
\item We extensively evaluate the performance of ISM in solving KDAC with synthetic and real data under various clustering quality and cost measures.  Our results show an improvement in execution time by up to a factor of roughly $70$ and $10^5$, compared to SM and DG, respectively. At the same time, ISM  outperforms SM and DG in clustering quality measures along with significantly lower computational cost.


\end{packeditemize}
\noindent\textbf{Related Work.}
There exist two general modes of discovering alternative clusterings  -- simultaneously or iteratively.  Simulataneous approaches find the multiple alternative clusterings at the same time~\cite{caruana2006meta,jain1999data,dang2010generation,dasgupta2010mining,mansinghka2009cross, niu2012nonparametric,poon2010variable}.  Iterative approaches find an alternative clustering given existing clustering~\cite{cui2010learning}.  Since this work focuses on the iterative paradigm, we elaborate on the related work along these lines.  Alternative clustering methods differ in how they measure novelty and cluster quality.  
Gondek and Hofmann \cite{gondek2007non} find an alternative clustering by conditional information (CI) bottleneck. Bae and Bailey \cite{bae2006coala} perform agglomerative clustering with cannot-link constraints imposed on the data points that belong together in the existing clustering. Cui et al.~\cite{cui2007non} find an alternative clustering by projecting the data to a subspace orthogonal to the existing clustering. Qi and Davidson \cite{qi2009principled} search for novelty by minimizing the Kullback-Leiber (KL) divergence between the original data and the transformed data subject to the constraint that the sum-squared-error between samples in the projected space with the existing clusters is small. 
Dang and Bailey \cite{dang2010hierarchical} find quality clusters by maximizing the mutual information (MI) between the alternative clusters and the data while simultaneously ensuring novelty by minimizing the MI between  alternative and existing clusterings.  

KDAC~\cite{niu2014iterative} discovers an alternative clustering by maximizing for cluster quality based on the spectral clustering objective and at the same time maximizing for novelty based on a non-linear dependence measure, HSIC~\cite{gretton2005measuring}, on the projected subspace of the alternative clustering. KDAC's ability to detect arbitrarily-shaped clusters is due to its use of the Hilbert-Schmidt Independence Criterion (HSIC)~\cite{gretton2005measuring} as a cluster quality measure. HSIC is motivated by the objective function of spectral clustering.  Moreover, since HSIC models  non-linear dependence, it is also utilized by KDAC to measure novelty.
In contrast, e.g., the orthogonal subspace projection approach in \cite{cui2007non} is limited, as it only captures linear dependencies.   Other approaches, such as \cite{gondek2007non,dang2010hierarchical}, can take non-linear dependencies into account by utilizing information theoretic measures.  However, doing so requires estimating joint probability distributions.  The advantage of KDAC over such approaches is that it utilizes HSIC for measuring novelty and cluster quality, which can capture non-linear dependencies through kernels, without having to explicitly learn the joint probability distributions; empirically, it significantly outperforms aforementioned schemes in clustering quality~\cite{niu2014iterative}.

\section{Kernel Dimension Alternative Clustering (KDAC) }
\label{gen_inst}

In alternative clustering, a dataset is provided along with existing clustering labels. Given this
as input, we seek a $\emph{new}$ clustering that is (a) distinct from the existing clustering, and (b) has high quality with respect to a clustering quality measure.
An example illustrating this is shown in Figure   \ref{fig:moon_illustrate}.  
This dataset comprises 400  points in $\mathcal{R}^4$. Projected to the first two dimensions, the dataset contains two clusters of intertwining parabolas shown as Clustering A. Projected to the last two dimensions, the dataset contains two Gaussian clusters shown as Clustering B. Points clustered together in one view can be in different clusters in the alternative view. In alternative clustering, given (a) the dataset, and (b) one of the two possible clusterings (e.g., Clustering B), we wish to discover the alternative clustering illustrated by the different view.

\begin{figure}[ht] 
  \centering
  \includegraphics[width=7cm,height=3cm]{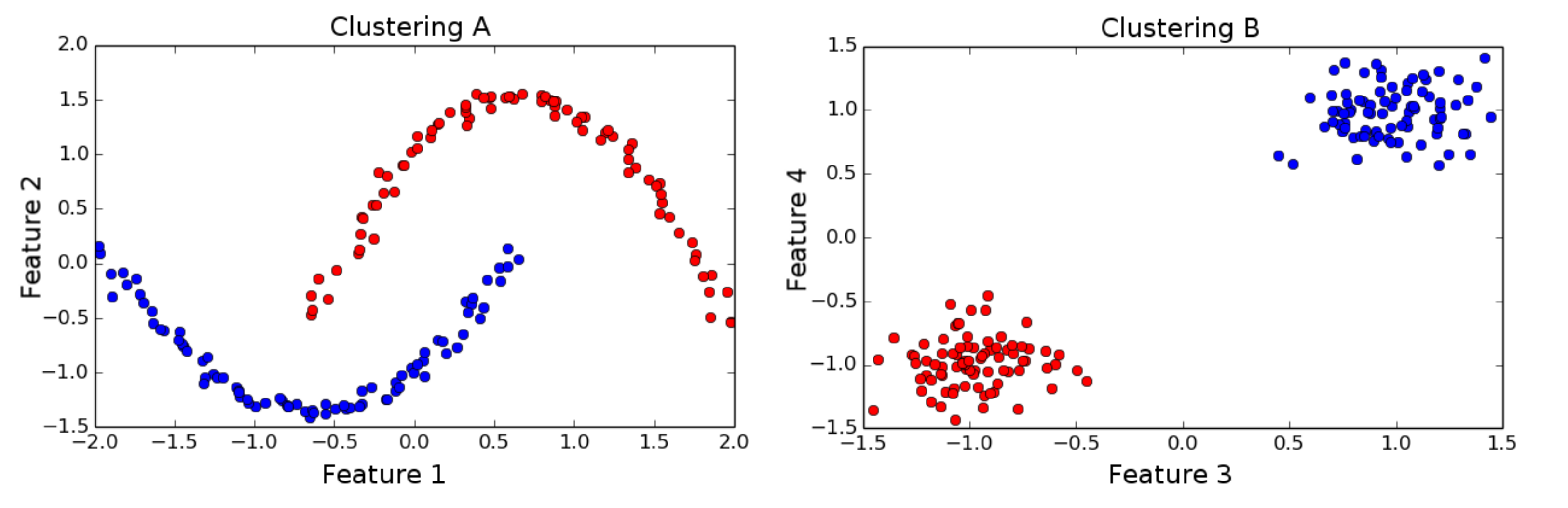}
  \caption{Four-dimensional moon dataset. Projection into the first two dimensions reveals different clusters than projection to the latter two dimensions.}
  \label{fig:moon_illustrate}
\end{figure}

Formally, let $X \in \mathcal{R}^{n \times d}$ be a dataset with $n$ samples and
$d$ features, along with an existing clustering 
$Y \in \mathcal{R}^{n \times k}$, where $k$ is
the number of clusters. If  $x_i$ belongs
to cluster $j$, then $Y_{i, j}= 1$; otherwise, $Y_{i,j}=0$.
We wish to discover an alternative clustering $U \in \mathcal{R}^{n \times k}$ 
on some lower dimensional subspace of dimension $q \ll d$. 
Let $W \in \mathcal{R}^{d \times q}$ be a projection
matrix such that $ X W \in \mathcal{R}^{n \times q}$.

We seek the optimal projection $W$ and clustering  $U$ that 
maximizes the statistical dependence between $X W$ with $U$, yielding a high clustering quality, while minimizing the
dependence between $X W$ and $Y$, ensuring the novelty of the new clustering. 
Denoting DM as a Dependence Measure function, and using $\lambda$ as a weighing constant, this optimization can be written as:  
\begin{subequations}\label{eq:orig_objective}
  \begin{align}
\text{Maximize:}& \quad\tmop{DM} ( X W \nocomma, U) - \lambda \tmop{DM} ( X W, Y),\\
   \text{s.t :}& \quad W^T W = I, U^T U = I.
\end{align}
\end{subequations}  
As in spectral clustering, the labels of the alternative clustering are retrieved by performing $K$-means on matrix $U$, treating its rows as samples. 
There are many potential choices for DM. The most well-known measures are correlation and mutual information (MI). While correlation performs well in many applications, it lacks the ability to measure non-linear relationships. Although there is clear relationship in Clustering A in Figure \ref{fig:moon_illustrate}, correlation would mistakenly yield a value of nearly 0. 
As a dependence measure, MI is superior in that it also measures non-linear relationships. However, due to the probabilistic nature of its formulation, a joint distribution is required. Depending on the distribution, the computation of MI can be prohibitive. 

For these reasons, the Hilbert Schmidt Independence Criterion (HSIC) \cite{gretton2005measuring} has been proposed for KDAC \cite{niu2014iterative}. Like MI, it captures non-linear relationships. Unlike MI, HSIC does not require estimating a joint distribution, and it relaxes the need to discretize continuous variables. In addition, as shown by Niu et al.~\cite{niu2014iterative}, 
HSIC is mathematically equivalent to spectral clustering, further implying that a high HSIC between the data and $U$ yields  high clustering quality.
A visual comparison of HSIC and correlation  can be found in
Figure \ref{HSIC_capture_nonlinear} of Appendix \ref{App:HSIC} in the supplement.

Using HSIC as a dependence measure, the objective of KDAC becomes 
\begin{subequations}\label{eq:main_objective}
  \begin{align}
\text{Maximize:}& \quad\tmop{HSIC} ( X W \nocomma, U) - \lambda \tmop{HSIC} ( X W, Y),\\
   \text{subject to:}& \quad W^T W = I, U^T U = I.
\end{align}
\end{subequations}  
where 
$\tmop{HSIC} ( X, Y) \equiv \frac{1}{(n-1)^2} \tmop{Tr} ( K_{X} H K_Y H).$
Here, the variables $K_{X}$ and $K_Y$ are Gram matrices, 
and the $H$ matrix is a
centering matrix where $H = I - \frac{1}{n} \bm{1}_n \bm{1}^T_n$ with 
$\bm{1}$ the $n$-sized  vector of all ones. 
The elements of $K_{X}$ and $K_Y$
are calculated by kernel functions $k_{X} ( x_i, x_j)$ and $k_Y ( y_i, y_j)$.  
 The kernel functions for $Y$ and $U$ used in KDAC are $K_Y = Y Y^T$ and $K_U = U U^T$, and the kernel function for $XW$ is  the Gaussian $k_{XW} ( x_i, x_j) = \exp(- {\tmop{Tr} [ ( x_i - x_j)^T
W W^T ( x_i - x_j)]}/{(2 \sigma^2)})$. 
Due to the equivalence of HSIC and spectral clustering, the practice of normalizing the kernel $K_{XW}$ is adopted from spectral clustering by Niu et al.~\cite{niu2014iterative}.   
That is, for $K_{XW}$ the unnormalized
Gram matrix, the normalized matrix is defined as $D^{- 1 / 2} K_{XW} D^{- 1 / 2}$ where $D=\mathrm{diag}(\bm{1}_n^TK_{XW})$ is a diagonal matrix whose elements are the  column-sums of $K_{XW}$.

\textbf{KDAC Algorithm.}
The optimization problem \eqref{eq:main_objective} is non-convex. The KDAC algorithm solves (\ref{eq:main_objective}) using alternate maximization between the variables $U$, $W$ and $D$, updating each while holding the other two fixed. After convergence, motivated by spectral clustering, $U$ is discretized via $K$-means to provide the alternative clustering. The algorithm proceeds in an iterative fashion,  summarized in Algorithm \ref{KDAC_algorithm}. In each iteration, variables $D$, $U$, and $W$ are updated as follows:

\textbf{Updating D:} 
While holding $U$ and $W$ constant, $D$ is computed as  
$D=\mathrm{diag}(\bm{1}_n^TK_{XW})$.
Matrix $D$ is subsequently treated as a scaling constant throughout the rest of the iteration. 

\textbf{Updating U:} 
Holding $W$ and $D$ constant and solving for $U$, (\ref{eq:main_objective}) reduces to :
\begin{align} \label{eq:spectral_clustering}
     \textstyle\max_{U:U^TU=I}  \tmop{Tr} ( U^T \mathcal{Q} U), 
\end{align}
where  $\mathcal{Q}= H D^{- 1 / 2} K_{XW} D^{- 1 / 2} H$. This is precisely spectral clustering~\cite{von2007tutorial}: 
 \eqref{eq:spectral_clustering} can be solved by setting $U$'s columns to the $k$ most dominant eigenvectors of $\mathcal{Q}$, which can be done in $O(n^3)$ time.

\begin{algorithm}[t]
    \scriptsize
    \SetKwInOut{Input}{Input}
    \SetKwInOut{Output}{Output}
    \Input{dataset $X$, original clustering $Y$}
    \Output{alternative clustering $U$ }
    Initialize $W_0$ using $W_{\mathrm{init}}$ from (\ref{eq:winit})\\
    Initialize $U_0$ from original clustering\\
    Initialize $D_0$ from $W$ and original clustering\\
    \While{($U$ not converged) or ($W$ not converged)}{
        Update $D$ \\
        Update $W$ by solving Equation (\ref{eq:main_cost_function})\\
        Update $U$ by solving Equation (\ref{eq:spectral_clustering})\\
	}
	Clustering Result $\gets$ Apply K-means to $U$
	\caption{KDAC Algorithm} 
\label{KDAC_algorithm}
\end{algorithm}
\begin{algorithm}[t]
    \scriptsize
    \SetKwInOut{Input}{Input}
    \SetKwInOut{Output}{Output}
    \Input{$U$,$D$,$X$, $Y$}
    \Output{ $W^*$ }
    Initialize $W_0$ to the previous value of $W$ in the master loop of KDAC.\\
	\While{$W$ not converged}{
		$W \gets \mathop{\mathrm{eig}}_{\min} ( \Phi ( W));$\\
	} 	
	\caption{ISM Algorithm} \label{masteralg}   
\end{algorithm}

\textbf{Updating W:} 
While holding $U$ and $D$ constant to solve for $W$, (\ref{eq:main_objective}) reduces to:
\begin{subequations}
\label{eq:main_cost_function}
\begin{align} 
   \text{Minimize:}\quad&F(W) =- \textstyle\sum_{i, j} \gamma_{i, j} e^{- \frac{\tmop{Tr} [ W^T A_{i, j}
   W]}{2 \sigma^2}}\label{eq:mainobj}\\
   \text{subject to:}\quad &W^TW=I\label{eq:mainconstr}
\end{align}
\end{subequations}
where $\gamma_{i,j}$ are the elements of matrix $\gamma = D^{- 1 / 2} H ( U U^T - \lambda Y Y^T) H D^{- 1 / 2}$, and 
$A_{i, j} = ( x_i - x_j) ( x_i - x_j)^T$ (see Appendix \ref{Cost_Derivation} in the supplement for the derivation). 
This objective, along with a Stiefel Manifold constraint, $W^TW=I$, pose a challenging optimization problem as neither is convex. Niu et al.~\cite{niu2014iterative} propose solving \eqref{eq:main_cost_function} through an algorithm termed Dimensional Growth (DG). This algorithm solves for $W$ by computing individual columns of $W$ separately through gradient descent (GD). Given a set of computed columns, the next column is computed by GD projected to a subspace orthogonal to the span of computed set. Since DG is based on GD, the computational complexity is dominated by computing the gradient of \eqref{eq:mainobj}. The latter is given by:
\begin{equation} \label{eq:gradient_of_main_cost_function}
\nabla F(W) =\textstyle\sum^n_i \sum_j^n  \frac{\gamma_{i, j}}{\sigma^2} e^{- \frac{\tmop{Tr} [
   W^T A_{i, j} W]}{2 \sigma^2}} A_{i, j} W. 
\end{equation}
%
%
%
The complexity of DG is $O(t_{DG} n^2d^2q)$, where $n$, $d$ are the dataset size and dimension, respectively, $q$ is the dimension of the subspace of the alternative clustering, and $t_{DG}$ is the number of iterations of gradient descent. The calculation of the gradient contributes the term $O(n^2d^2q)$. Although this computation is highly parallelizable, the algorithm still suffers from slow convergence rate. Therefore, $t_{DG}$ often dominates the computation cost.  


An alternative approach to optimize \eqref{eq:main_cost_function} is through classic methods for performing optimization on the Stiefel Manifold (SM) \cite{wen2013feasible}. The computational complexity of this algorithm is dominated by the computation of the gradient and a matrix inversion with $t_{SM}$ iterations.   This yields a complexity of $O(t_{SM} n^2 d^2 + t_{SM}d^3)$ for SM. Finally, as gradient methods applied to a non-convex objective, both SM and DG require multiple executions from random initialization points to find improved local minima. This approach becomes less effective as the dimension $d$ increases. 

\section{An Iterative Spectral Method}

The computation of KDAC is dominated by the $W$ updates in Algorithm~\ref{KDAC_algorithm}. Instead of using DG or SM to solve the optimization problem for $W$ in KDAC, we propose an Iterative Spectral Method (ISM).  Our algorithm is motivated from the following observations. The Lagrangian of \eqref{eq:main_cost_function} is:
 \begin{align*}
      \mathcal{L} ( W, \Lambda) =& - \textstyle\sum_{i, j} \gamma_{i, j} \exp\left(- \frac{\tmop{Tr}
      (W^T A_{i, j} W)}{2 \sigma^2}\right)\\
      &- \frac{1}{2} \tmop{Tr} ( \Lambda ( W^T W -
      I)) \numberthis \label{eq:lag}
\end{align*}
    Setting $\nabla_W \mathcal{L}(W,\Lambda)=0$ gives us the equation:
\begin{align}\Phi(W)W = W \Lambda,\label{eq:balance} \end{align}
where 
\begin{equation} \label{eq:phi_equation}
  \Phi ( W) = \textstyle\sum_{i, j} \frac{\gamma_{i, j}}{\sigma^2} \exp(- \frac{\tmop{Tr} [ W^T A_{i, j}
  W]}{2 \sigma^2}) A_{i, j},
\end{equation}
and $\Lambda$ is a diagonal matrix.
Recall that a feasible $W$, satisfying \eqref{eq:mainconstr}, is orthonormal. 
\eqref{eq:balance} is an eigenequation; thus, a stationary point $W$ of the Lagrangian \eqref{eq:lag}  comprises of $q$ eigenvectors of $\Phi(W)$ as columns.
Motivated by this observation, ISM attempts to find such a $W$ in the following iterative fashion. Let $W_0$ be an initial  matrix. Given $W_k$ at iteration $k$, the matrix $W_{k+1}$  is computed as:
$$W_{k+1} = \textstyle\mathop{\mathrm{eig}}_{\min} ( \Phi ( W_k)), \quad k=0,1,2,\ldots,$$
where the operator $\mathop{\mathrm{eig}}_{\min} (A)$ returns a matrix whose columns are
the $q$ eigenvectors corresponding to the smallest eigenvalues of $A$.


ISM is summarized in Alg.~\ref{masteralg}. Several important observations are in order. First, the algorithm ensures that $W_k$, for $k\geq 1$, is feasible, by construction: selected eigenvectors are orthonormal and satisfy \eqref{eq:mainconstr}. Second, it is also easy to see that a fixed point of the algorithm will also be a stationary point of the Lagrangian \eqref{eq:lag} (see also~Lemma~\ref{lemma:eig}). Though it is harder to prove,  selecting  eigenvectors corresponding to the \emph{smallest} eigenvalues is key: we show that this is precisely the property that relates a fixed point of the algorithm to the local minimum  conditions (see Thm.~\ref{thm:stationary}).
Finally, ISM has several computational advantages. For $t_{ISM}$ iterations, the calculation of $\Phi(W)$, and the ensuing eigendecomposition yields a complexity of $O(t_{ISM}( n^2 d^2 +  d^3))$. Since $q\ll d$, various approximation methods \cite{vladymyrov2016variational}\cite{richtarikgeneralized}\cite{lei2016coordinate} can be employed to find the few eigenvectors. For example, the Coordinate-wise Power Method\cite{lei2016coordinate}, approximates the most dominant eigenvalue at $O(d)$ time, reducing ISM's complexity to $O(t_{ISM} n^2 d^2)$. This improvement is further confirmed experimentally (see Figure \ref{fig:ND_vs_time}). Lastly, $t_{ISM}$ is magnitudes smaller than both $t_{DG}$ and $t_{SM}$. In general $t_{ISM} < 10$, while $t_{SM} > 50$ and $t_{DG} > 200$.

\subsection{Convergence Guarantees}

As mentioned above, the selection of the eigenvectors corresponding to the \emph{smallest} eigenvalues of $\Phi(W_k)$ is crucial for the establishment of a stationary point. Namely, we establish the following theorem:
\begin{theorem}\label{thm:stationary}
  For large enough $\sigma$ (satisfying Inequality \eqref{ineq:large_sigma_orig}), a fixed point $W^*$ of  Algorithm~\ref{masteralg} satisfies the necessary conditions of a local minimum of \eqref{eq:main_cost_function} if $\Phi(W^*)$ is full rank.
\end{theorem}
\begin{proof}
The main body of the proof is organized into a series of lemmas proved in the supplement.
Our first auxiliary lemma (from~\cite{wright1999numerical}), establishes conditions necessary for a stationary point of the Lagrangian to  constitute local minimum. 
\begin{lemma} \label{lemma:2nd_order}
  [Nocedal,Wright, Theorem 12.5~{\cite{wright1999numerical}}] (2nd Order Necessary Conditions)
    Consider the optimization problem:
  $ \min_{W : h (W) = 0} f (W), $
where $f : \mathbb{R}^{d \times q} \to \mathbb{R}$ and $h :
  \mathcal{R}^{d \times q} \to \mathbb{R}^{q \times q}$ are twice continuously
  differentiable. Let   $\mathcal{L}$ be the Lagrangian of this optimization problem. Then, a local minimum must satisfy the following  conditions:
  \begin{subequations}
    \begin{align} 
    &\nabla_W \mathcal{L} (W^{\ast}, \Lambda^{\ast}) = 0, \label{eq:1st_W}\\
    &\nabla_{\Lambda} \mathcal{L} (W^{\ast}, \Lambda^{\ast}) = 0, \label{eq:1st_lambda}\\
    \begin{split}
    \tmop{Tr} ( Z^T &\nabla_{W W}^2 \mathcal{L}(W^{\ast}, \Lambda^{\ast}) Z) \geq 0 \\&\tmop{for}   \tmop{all} Z \neq 0 , \tmop{with} 
       \nabla h (W^{\ast})^T Z = 0.   \label{eq:2nd_W}
       \end{split}
    \end{align}   
  \end{subequations}
\end{lemma}
Armed with this result, we next characterize the properties of a fixed point of Algorithm \ref{masteralg}:
\begin{lemma} \label{basic_lemma}
  Let $W^{\ast}$ be a fixed point of Algorithm \ref{masteralg}. Then it satisfies:
   $ \Phi ( W^{\ast}) W^{\ast} = W^{\ast} \Lambda^{\ast},$
  where $\Lambda^{\ast} \in \mathcal{R}^{q \times q}$ is a diagonal matrix
  containing the $q$ smallest eigenvalues of $\Phi ( W^{\ast})$ and
  $W^{\ast^T} W^{\ast} = I. $
\end{lemma}
The proof can be found in Appendix \ref{proof_of_lemma_2}.
Our next result, whose proof is  in Appendix \ref{proof_of_lemma_3}, states that 
 a fixed point satisfies the 1st order conditions of Lemma \ref{lemma:2nd_order}.  
\begin{lemma} \label{lemma:eig}
  If $W^{\ast}$ is a fixed point and $\Lambda^{\ast}$ is as defined in Lemma~\ref{basic_lemma},
  then $W^{\ast}$, $\Lambda^*$ satisfy the 1st order conditions (\ref{eq:1st_W})(\ref{eq:1st_lambda}) of Lemma \ref{lemma:2nd_order}.
\end{lemma}
Our last lemma, whose proof is in Appendix \ref{proof_of_lemma_4}, establishes that a fixed point satisfies the 2nd order conditions of Lemma \ref{lemma:2nd_order}, for large enough $\sigma$.  
\begin{lemma} \label{lemma:2nd_order_lemma}
  If $W^{\ast}$ is a fixed point, $\Lambda^{\ast}$ is as defined in Lemma~\ref{basic_lemma}, and $\Phi(W^*)$ is full rank,
  then given a large enough $\sigma$ (satisfying Inequality \eqref{ineq:large_sigma_orig}), $W^{\ast}$ and $\Lambda^*$ satisfy the 2nd order condition (\ref{eq:2nd_W}) of Lemma \ref{lemma:2nd_order}.
\end{lemma}
 Thm.~\ref{thm:stationary} therefore follows.
\end{proof}

Thm.~\ref{thm:stationary} is stated in terms of a large enough $\sigma$; we can characterize this constraint precisely. In the proof of Lemma~\ref{lemma:2nd_order_lemma} we establish the following condition on $\sigma$:
\begin{multline}
   \sigma^2 [\min_i ( \bar{\Lambda^*}_i) -\max_j(\Lambda_j^*)] \geq \\
    \sum_{i, j} \frac{|\gamma_{i, j}|}{\sigma^2} e^{-
   \frac{\tmop{Tr} ( ( W^{*^T} A_{i, j} W^* \nobracket)}{2 \sigma^2}}\tmop{Tr} (A^T_{i,j}A_{ij}). \label{ineq:large_sigma_orig}
\end{multline}
%
%
Here, $\Lambda^*$ is the set of $q$ smallest eigenvalues of $\Phi(W)$, and $\bar{\Lambda^*}$ is the set of the remaining eigenvalues. The left-hand side (LHS) of the equation further motivates ISM's choice of  eigenvectors corresponding to the $q$ smallest eigenvalues. This selection guarantees that the LHS of the inequality is positive. Therefore, given a large enough $\sigma$, Inequality \eqref{ineq:large_sigma_orig} and the 2nd order condition is satisfied. 

Furthermore, this equation provides a reasonable suggestion for the value of $q$.  Since we wish to maximize the term $(\min_i ( \bar{\Lambda^*}_i) -\max_j(\Lambda_j^*))$ to satisfy the inequality, the value $q$ should be set where this gap is maximized. More formally, we will defined
\begin{align}
   \delta_{gap}=\min_i ( \bar{\Lambda^*}_i) -\max_j(\Lambda_j^*)
   \label{eq:eiggap}.
\end{align}
as the eigengap.



\subsection{Spectral Initialization via Taylor Approximation}
\label{sec:initialization}
ISM admits a natural initialization point, constructed via a Taylor approximation of the objective. As we show experimentally in Section~\ref{sec:exp}, this initialization is a contribution in its own right: it improves both clustering quality and convergence time for ISM \emph{as well as} competitor algorithms.  To obtain a good initialization, observe that by using the 2nd order Taylor approximation of the objective function \eqref{eq:mainobj} at $W=0$, the Lagrangian can be approximated by 
\begin{align*}
\tilde{\mathcal{L}}(W,\Lambda) \approx & -\textstyle \sum_{i, j} \gamma_{i, j}  \left( 1 - \frac{\tmop{Tr}
   (W^T A_{i, j} W)}{2 \sigma^2} \right)\\
   & + \frac{1}{2} \tmop{Tr} (\Lambda (I -
   W^T W)). 
\end{align*}
  Setting $\nabla_W \tilde{\mathcal{L}}(W,\Lambda)=0$ reduces the problem into a simple eigendecomposition, namely, the one defined by the system 
$\left[ \textstyle \sum_{i, j} \frac{\gamma_{i, j}}{\sigma^2} A_{i,
   j} \right] W = W \Lambda .$
Hence, the 2nd order Taylor approximation of the original cost objective has a closed form global minimum that can be used as an initialization point, namely:
\begin{align}
   \textstyle W_{\mathrm{init}}=\mathrm{eig}_{\min}( \sum_{i, j} {\gamma_{i, j}} A_{i,
   j}/{\sigma^2})\label{eq:winit}.
\end{align}
We use this spectral initialization (SI) in the first master iteration of KDAC. In subsequent master iterations, $W_0$ (the starting point of ISM) is set to be the last value to which ISM converged to previously.

\section{Experimental Results}\label{sec:exp}
We experimentally validate the performance of ISM in terms of both speed and clustering quality. The source code is publicly available at github\footnote{https://github.com/neu-spiral/ISM}. Because \cite{niu2014iterative} has already performed extensive comparisons of KDAC against other alternative clustering methods, this section will concentrate on comparing  ISM to competing models for optimizing KDAC: Dimensional Growth (DG)~\cite{niu2014iterative} and gradient descent on the Stiefel Manifold (SM)~\cite{wen2013feasible}. SM is a generic approach, while DG is the approach originally proposed to solve KDAC~\cite{niu2014iterative}. 

In addition to introducing a new algorithm for optimizing KDAC, we also proposed an intelligent initialization scheme based on Taylor approximation in Section \ref{sec:initialization}, we call Spectral Initialization (SI).  We also investigate how SI effects the performance of the various algorithms compared to the standard random initialization (RI).

\begin{table}[h]
\tiny
\centering
\setlength\tabcolsep{0.2pt}
    \begin{tabular}{c|c|c|c|c|c}
    \textbf{SG} 
        & NMI       $\uparrow$
        & CQ        $\uparrow$
        & Novelty   $\downarrow$
        & Cost      $\downarrow$
        & Time      $\downarrow$\\
    \midrule
        ISM SI
        	& \textbf{1} 
        	& \textbf{2} 
        	& \textbf{0} 
        	& -1.2
        	& \textbf{0.02} \\
        ISM RI
        	& 0.4$\pm$0.49
        	& 1.6$\pm$0.478
        	& \textbf{0.0$\pm$0.0} 
        	& -1.48$\pm$0.381
        	& 0.04$\pm$0.01\\
        SM SI
        	& \textbf{1} 
        	& 1.99
        	& \textbf{0} 
        	& \textbf{-1.791} 
        	& 0.404\\
        SM RI
        	& \textbf{1$\pm$0.0} 
        	& 1.79$\pm$0.398
        	& \textbf{0.0$\pm$0.0} 
        	& -1.59$\pm$0.398
        	& 1.47$\pm$1.521\\
        DG SI
        	& 0
        	& 1.629
        	& 1
        	& -1.316
        	& 1.732\\
        DG RI
        	& 0.93$\pm$0.196
        	& 0.99$\pm$0.0
        	& 0.03$\pm$0.104
        	& -0.99$\pm$0.0
        	& 1.51$\pm$0.102\\
    \midrule
    \textbf{LG} \\
        ISM SI
        	& \textbf{1} 
        	& \textbf{1.9}
        	& \textbf{0} 
        	& 509
        	& 0.239\\
        ISM RI
        	& 0.8$\pm$0.4
        	& \textbf{1.9$\pm$0.011}
        	& \textbf{0.0$\pm$0.0} 
        	& 605$\pm$193
        	& \textbf{0.18$\pm$0.038} \\
        SM SI
        	& \textbf{1} 
        	& 1.9
        	& \textbf{0} 
        	& 509
        	& 0.413\\
        SM RI
        	& 0.4$\pm$0.49
        	& 1.9$\pm$0.031
        	& 0.6$\pm$0.49
        	& 1095$\pm$479
        	& 11.21$\pm$11.01\\
        DG SI
        	& \textbf{1} 
        	& \textbf{1.9} 
        	& \textbf{0} 
        	& 509.975
        	& 1595.174\\
        DG RI
        	& 0.1$\pm$0.32
        	& 0.61$\pm$0.17
        	& 0.9$\pm$0.32
        	& \textbf{101.5$\pm$30.8} 
        	& 1688$\pm$551\\
     \midrule
    \textbf{Moon} \\       	
        ISM SI
        	& \textbf{1} 
        	& \textbf{2} 
        	& \textbf{0} 
        	& \textbf{149} 
        	& \textbf{0.575} \\
        ISM RI
        	& 0.4$\pm$0.4
        	& \textbf{2.0$\pm$0.0} 
        	& 0.5$\pm$0.5
        	& 269.98$\pm$43
        	& 3.07$\pm$1.4\\
        SM SI
        	& \textbf{1} 
        	& 1.998
        	& \textbf{0} 
        	& \textbf{149} 
        	& 2.653\\
        SM RI
        	& 0.27$\pm$0.335
        	& \textbf{2.0$\pm$0.0} 
        	& 0.5$\pm$0.5
        	& 277$\pm$18
        	& 258$\pm$374\\
        DG SI
        	& \textbf{1} 
        	& 1.998
        	& \textbf{0} 
        	& 149.697
        	& 359.188\\
        DG RI
        	& 0.09$\pm$0.3
        	& 1.27$\pm$0.6
        	& 0.8$\pm$0.4
        	& 161.65$\pm$102
        	& 3212$\pm$1368\\
    \midrule
    \textbf{MoonN} \\       	       	
        ISM SI
        	& \textbf{1} 
        	& \textbf{2} 
        	& \textbf{0} 
        	& \textbf{15.3} 
        	& \textbf{0.3} \\
        ISM RI
        	& 0.91$\pm$0.2
        	& \textbf{2.0$\pm$0} 
        	& 0.04$\pm$0.1
        	& 15.59$\pm$0.8
        	& 0.55$\pm$0.2\\
        SM SI
        	& \textbf{1} 
        	& \textbf{2} 
        	& \textbf{0} 
        	& \textbf{15.3} 
        	& 0.452\\
        SM RI
        	& 0.22$\pm$0.4
        	& \textbf{2.0$\pm$0} 
        	& 0.72$\pm$0.4
        	& 17.7$\pm$1
        	& 102$\pm$48.5\\
        DG SI
        	& \textbf{1} 
        	& \textbf{2} 
        	& \textbf{0} 
        	& 15.4
        	& 3836.3\\
        DG RI
        	& 0$\pm$0.0
        	& 1.6$\pm$0.04
        	& 1.0$\pm$0.0
        	& 16.2$\pm$0.06
        	& 3588$\pm$304\\
    \midrule
    \textbf{Flower} \\       	
        ISM SI
        	& \textbf{-} 
        	& \textbf{0.41} 
        	& \textbf{0} 
        	& 20
        	& \textbf{0.04}\\
        ISM RI
        	& \textbf{-} 
        	& \textbf{0.41$\pm$0.0} 
        	& \textbf{0.0$\pm$0.0} 
        	& 19.51$\pm$0.0
        	& 0.1$\pm$0.01 \\
        SM SI
        	& \textbf{-} 
        	& \textbf{0.41} 
        	& \textbf{0} 
        	& 20
        	& 0.153\\
        SM RI
        	& \textbf{-} 
        	& 0.3$\pm$0.2
        	& 0.01$\pm$0.012
        	& \textbf{0.7$\pm$0.5} 
        	& 0.8$\pm$0.6\\
        DG SI
        	& \textbf{-} 
        	& \textbf{0.41} 
        	& \textbf{0} 
        	& 20
        	& 27.251\\
        DG RI
        	& \textbf{-} 
        	& 0.35$\pm$0.02
        	& 0.2$\pm$0.4
        	& 20.6$\pm$1
        	& 37.0$\pm$5\\
    \midrule
    \textbf{Faces} \\       	       	       	
        ISM SI
        	& \textbf{0.57} 
        	& \textbf{0.61} 
        	& 0.004
        	& 59.6
        	& \textbf{1.5} \\
        ISM RI
        	& 0.56$\pm$0.002
        	& \textbf{0.61$\pm$0.0} 
        	& \textbf{0.0$\pm$0.0} 
        	& 59.6$\pm$0.118
        	& 1.51$\pm$0.146\\
        SM SI
        	& 0.562
        	& 0.608
        	& 0.004
        	& 59.559
        	& 7.96\\
        SM RI
        	& \textbf{0.57$\pm$0.002} 
        	& \textbf{0.61$\pm$0} 
        	& 0.004$\pm$0
        	& 59.7$\pm$0.1
        	& 109$\pm$35.2\\
        DG SI
        	& 0.564
        	& 0.6
        	& 0.004
        	& 59.8
        	& 100591.071\\
        DG RI
        	& 0.458$\pm$0.045
        	& 0.565$\pm$0.018
        	& 0.024$\pm$0.018
        	& \textbf{54.28$\pm$3} 
        	& 166070$\pm$4342\\
    \midrule
    \textbf{WebKb} \\
        ISM SI
        	& \textbf{0.37} 
        	& 0.286
        	& \textbf{0} 
        	& -0.273
        	& 231\\
        ISM RI
        	& 0.29$\pm$0.07
        	& 0.54$\pm$0.2
        	& 0.01$\pm$0.002
        	& -0.52$\pm$0.2
        	& 137.45$\pm$15.7\\
        SM SI
        	& 0.048
        	& \textbf{3.296} 
        	& 0.034
        	& \textbf{-3.159} 
        	& \textbf{9.945} \\
        SM RI
        	& 0.33$\pm$0.06
        	& 0.12$\pm$0.005
        	& 0.008$\pm$0.004
        	& -0.11$\pm$0.004
        	& 13511$\pm$13342\\
        DG SI
        	& 0.048
        	& 1.066
        	& 0.034
        	& -1.019
        	& 199887.134\\
        DG RI
        	& 0.23$\pm$0
        	& 1.06$\pm$0
        	& 0.1$\pm$0.005
        	& 0.616$\pm$0.04
        	& 727694$\pm$41068\\    
\end{tabular}
\caption{
 The Normalized Mutual Information, Clustering Quality, and clustering novelty are abbreviated in the first 3 columns as NMI, CQ, and Novelty. The cost of the objective and run time is displayed in the last two columns. For each optimization technique, spectral initialization (SI) and random initialization (RI) are separately tested. With RI, 10 random initial points have been tested with their mean and std displayed. } \label{experiment_table_and_figures}
\end{table}

\begin{figure}[!t] 
  \centering
  \includegraphics[width=7.7cm,height=8.7cm]{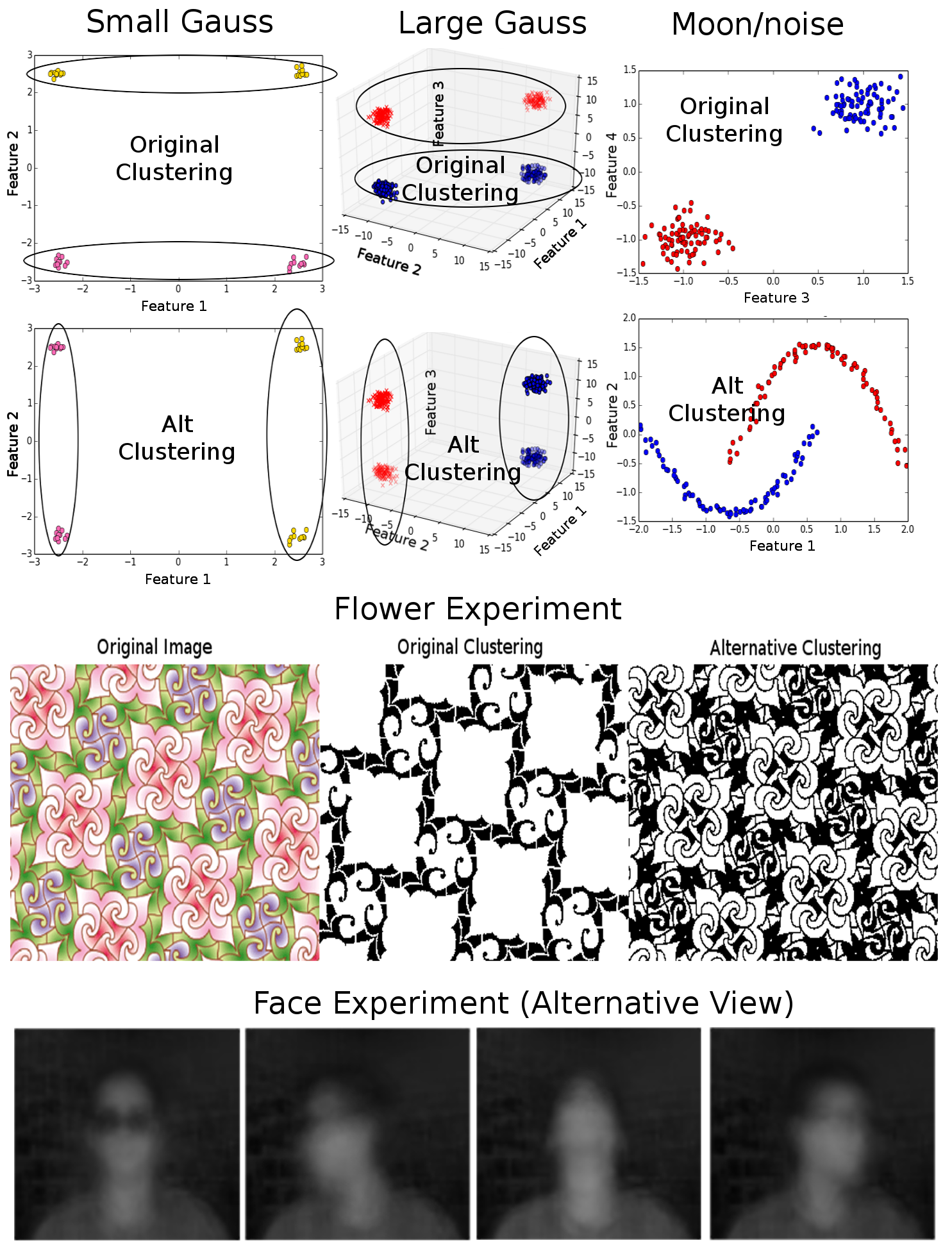} 
  \caption{ Figure includes all original and alternative displayable clusterings. The face data is originally clustered by the identity of the individual. Here the average image of the alternative clustering is displayed. It is observable that the original and alternative clusters are all visually obvious clusters and provide alternative views.} 
  \label{fig:results_figure}
\end{figure}

\begin{figure}[!t] 
  \centering
  \includegraphics[width=0.7\columnwidth]{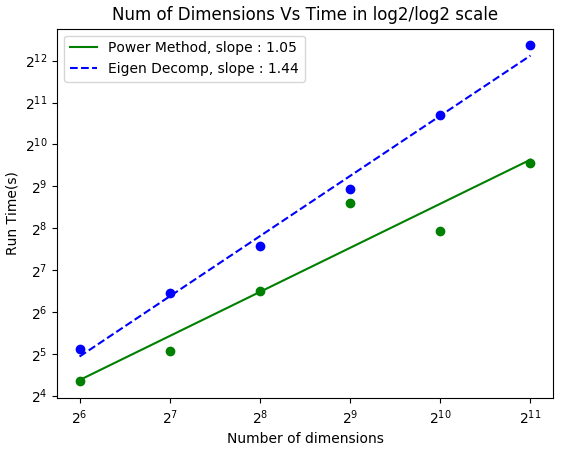}
  \caption{Growth in dimension vs time in log/log scale. A slope of 1 is linear growth.}
  \label{fig:ND_vs_time}
\end{figure}

\noindent\textbf{Datasets.} We perform experiments on four synthetic and three real datasets.  The synthetic data are displayed in Figure \ref{fig:results_figure}. Small Gaussian (SG) contains four Gaussian clusters with $40$ samples and two features shown at the top left of the figure. When this data is projected down to feature 2, the original clustering is created. Rotating this projection onto feature 1, the alternative clustering emerges. Large Gaussian (LG) contains four Gaussian clusters with $1000$ samples and four dimensions is shown in the top center location. The Gaussian clusters are rotated by $45$ degrees to reside in 3D, and the fourth dimension is generated from a uniform noise distribution. This synthetic data is designed to test the response of KDAC within a noisy environment.  
We generate two additional synthetic datasets shown in top right: Moon and Moon+Noise (MoonN). Both datasets have the first two dimensions as two parabolas and the second two dimensions as Gaussian clusters. The MoonN dataset further includes three noisy dimensions generated from a uniform distribution with 1000 samples.
Since the Gaussian clusters have a compact structure and the parabolas have a non-linear structure, these datasets demonstrate KDAC's ability to handle mixed clustering structures in a uniformly noisy environment.
Due to the novelty of alternative clustering, there has been very few public repository data that have at least two alternative labels. 
The two traditional benchmark datasets are CMU's WebKB dataset~\cite{cmu4universities} and face images from the UCI KDD repository~\cite{bay2000uci}.  CMU's WebKB dataset consists of 1041 html pages from 4 universities. One labelling is based on universities, and an alternative labelling based on topic (course, faculty, project, student). After preprocessing by removing rare and stop words, we are left with 500 words. 
The face dataset consists of $640$ images from $20$ people in four different poses. Each image has 32x30 pixels. Images are vectorized and PCA is then used to further condense the dataset by keeping 85\% of the variance, resulting in a dataset of 624 samples and 20 features. 
We set the existing clustering based on identity, and
seek an alternative clustering   based on pose. The last real dataset is the Flower image by Alain Nicolas \cite{particleoft}, a 350x256 pixel image. The RGB values of each pixel is taken as a single sample, with repeated samples removed. This results in a dataset of 256 samples and 3 features.
Although this dataset does not have labels, the quality of the alternative clustering can be visually observed. With the exception of the high dimensional WebKB data, the remaining datasets are visualized in Figure \ref{experiment_table_and_figures} (b).

\noindent\textbf{Hyperparameter Tuning.}
ISM has 3 hyperparameters that require tuning, $q$, $\lambda$, and $\sigma$. The hyperparameter $q$ with a potential range of (0,$d$) is initially set to $k$, the number of clusters. A grid search was then conducted for $\lambda \in (0,10]$ and $\sigma \in (0,10]$ to find the highest CQ ($\lambda$, $\sigma$) pair that satisfies inequality (\ref{ineq:large_sigma_orig}) at an increment of 0.01. In the event that Inequality (\ref{ineq:large_sigma_orig}) cannot be satisfied, the highest CQ closest to satisfying (\ref{ineq:large_sigma_orig}) is used. At this point, since $\Phi(W^*)$ is computed, the $q$ value that maximizes the eigengap as defined in Eq(\ref{eq:eiggap}) can be used if (\ref{ineq:large_sigma_orig}) is not yet satisfied. Note that competing models do not have a natural way of selecting hyperparameters. Since all competing models optimize the same objective, we report all results using the same hyperparameters. To ensure reproducible experiments, the hyperparameters utilized in each experiment is provided in Appendix \ref{App:Hyperparameters}.

\noindent\textbf{Evaluation Method.} 
To exhaustively evaluate ISM on KDAC, both external and internal measures have been recorded in Table~\ref{experiment_table_and_figures}. More specifically, column 1 and 3 (NMI and Novelty) are external measures because they compare alternative cluster assignments against an externally known ground truth. To make the comparison, Normalized Mutual Information (NMI) as suggested by \cite{strehl33knowledge} was used.  The NMI is a measure between 0 to 1 with 0 denoting no relationship and 1 as maximum relationship. Column 1 (NMI) is calculated by computing the NMI between the ground truth and the alternative cluster assignments. Column 3 (Novelty) is calculated by computing the NMI between the alternative cluster assignments against the original label. Ideally, we wish for the NMI of the alternative cluster assignments against the ground truth to be 1, and 0 against the original clustering.  If we let $U$ and $L$ be two clustering assignments, NMI can be calculated with $NMI(L,U) = \frac{I(L,U)}{\sqrt{H(L)H(U)}}$, where $I(L,U)$ is the mutual information between $L$ and $U$, and $H(L)$ and $H(U)$ are the entropies of $L$ and $U$ respectively. Lastly, since no ground truth exists for the Flower dataset, the NMI field is not applicable and indicated with a dash (-).

Columns 2, 4 and 5 (CQ, Cost, Time) are internal measures used to evaluate ISM. They are internal measures since they are computed solely from the data and the algorithm with no external knowledge applied during the comparison. The clustering quality (CQ) is computed with $HSIC(XW,U)$ with $U$ as the clustering solution. This is equivalent to the spectral clustering objective with high values denoting high clustering quality. The cost quality records the objective function of KDAC in Eq. (\ref{eq:main_objective}). The last column, Time, measures the execution time of the entire KDAC algorithm in seconds on an Intel Xeon E7 processor.

In Table~\ref{experiment_table_and_figures}, we report the performance of the various methods: our ISM, SM, and DG, paired with two initialization schemes: our spectral initialization (SI) and random initialization (RI).  For random initialization, we repeat these 10 times and report the mean and standard deviation for each measure.  The optimal direction of each measure is denoted by the $\uparrow\downarrow$, with $\uparrow$ denoting a preference towards larger values and $\downarrow$ otherwise. For each field, the optimal result is printed in bold font.

\noindent\textbf{Performance Comparison Results.} 
Table \ref{experiment_table_and_figures} illustrates that ISM with SI outperforms its competitors in both internal and external quality measure for the first 3 columns. Since it is possible for the objective cost to achieve a low cost with a trivial solution based on different $(\sigma, \lambda)$ pairs, the objective should be low, but it is not always indicative of better clustering quality. This is demonstrated in the case of WebKb dataset, where SM+SI achieved a lower cost with a faster convergence. However, upon inspecting the clustering allocation in this case, we observed that it was a trivial classification, with almost all points in the same cluster.  Since the demand for a faster KDAC was the original motivator, special attention should be paid to the Time column. The execution time of ISM significantly outperforms the original approach (DG) as well as the standard approach (SM). This speed improvement is especially true in the Faces dataset where ISM improves the speed by 5 folds.

\noindent\textbf{The Effect of Spectral Initialization.}
 By comparing SI against RI, we can isolate the effect of our proposed spectral initialization (SI). Comparing the rows, the spectral initialization improved both time and clustering quality for all methods. From this observation, we conclude that SI contributes to clustering quality by starting each algorithm at a desirable initialization. The convergence improvement came from placing each algorithm closer to the local minimum. Comparing all methods when using SI, we observed that ISM optimization still outperformed other algorithms in terms of time. From this, we conclude that the proposed SI has a greater impact on the clustering quality while the ISM optimization technique contribute towards faster convergence.

\noindent\textbf{Scalability.}
Note that KDAC consists of optimizing $U$ and $W$ from Eq. (\ref{eq:main_objective}). Since the computational bottleneck resided in the optimization of $W$, ISM was designed to speed up this portion. The scalability analysis, therefore, will concentrate on only the optimization of $W$. ISM has a computational complexity of $O(t_{ISM}(n^2d^2+d^3)$ and the complexity can be divided into the calculation of the derivative and the eigendecomposition of $\Phi(W)$. Although the derivative contributes to $O(n^2d^2)$, it is a highly parallelizable operation that can be rendered trivial through the usage of GPUs. Therefore, ISM's true bottleneck is not the number of samples ($n$), but the number of dimensions due to the $O(d^3)$ operation of eigendecomposition. As the dimensionality of the data increase, the complexity grows at a cubic rate. Yet, while this growth of dimensionality may exacerbate algorithms such as SM, the negative influence on ISM is limited. This is because the $O(d^3)$ term from SM came from a matrix inversion while ISM uses a spectral method. Since only very few eigenvalues are required, there exists many approximation algorithms to speed up the eigendecomposition \cite{vladymyrov2016variational} \cite{richtarikgeneralized}\cite{lei2016coordinate}.  To demonstrate this point, Coordinate-wise Power Method (CPM) \cite{lei2016coordinate} was implemented and compared to the eigendecomposition operation from Numpy. In this experiment, noisy dimensions of uniform distributions are added to the synthetic dataset of Large Gauss (LG) such that the total dimension increase in multiples of 2. As the dimension of the data grows exponentially, the time of execution is recorded in Figure \ref{fig:ND_vs_time} in log scale. Given a log/log scale, a slope of 1 is linear while a slope of 2 is quadratic. From studying the slope of each operation, the eigendecomposition from Numpy grows at a rate between linear and quadratic. However, by utilizing CPM, the growth becomes roughly linear. Therefore, the refinement of utilizing a spectral optimization method is a key reason for speed improvement.

\section{Conclusions}

We have proposed an iterative spectral optimization technique for solving a non-convex optimization problem while being constrained on a Stiefel manifold. This new technique demonstrates speed improvement for any algorithm that could be formulated into Eq (\ref{eq:main_cost_function}). The ISM algorithm is easy to implement with existing software libraries. Due to the usage of the spectral method, approximation techniques from existing research could be deployed to further speed up the eigendecomposition. Accompanied with the algorithm are the theoretical guarantees that satisfy the first and second order necessary conditions for a local minimum. Besides these guarantees, we also proposed a natural initialization point based on Taylor approximation of the original cost function. Experiments on synthetic and real data confirmed that our proposed spectral initialization improved the performance of all the optimization algorithms, ISM, SM and DG. Simultaneously, the experiments demonstrate that our proposed ISM algorithm had the best convergence rate compared to competing models with up to 5 folds of speed improvement. Although we focus this paper on an important application of ISM, alternative clustering, the optimization algorithm proposed and guarantees can be extended to other optimization problems involving Gaussian-kernel like objectives constrained over the Stiefel manifold, which is a common formulation for dimensionality reduction with a kernel-based objective. 
Understanding whether ISM can be applied to such objectives and be leveraged to solve broader classes of problems is a natural future direction for this work.

\subsubsection*{Acknowledgements}
We would like to acknowledge support for this project from the NSF grant IIS-1546428. We would also like to thank Yale Chang for his insightful discussions.

\begin{small}
\bibliographystyle{unsrt}
\bibliography{reference}
\end{small}

\newpage
\onecolumn




\begin{appendices}
\section{Derivation for Equation \ref{eq:main_cost_function} }
Given the objective function,
\[ \begin{array}{ll}
     \max & \tmop{HSIC} ( X W \nocomma, U) - \lambda \tmop{HSIC} ( X W, Y)\\
     U, W & \\
     s.t & W^T W = I \nocomma, U^T U = I.
   \end{array} \]
Using the HSIC measure defined, the objective function can be
rewritten as
\[ \begin{array}{lll}
     \tmop{HSIC} ( X W \nocomma, U) - \lambda \tmop{HSIC} ( X W, Y) & = &
     \tmop{Tr} ( H \nocomma U U^T H D^{\frac{-1}{2}} K_{XW} D^{\frac{-1}{2}}) - \lambda \tmop{Tr} ( H Y Y^T H
     D^{\frac{-1}{2}} K_{XW} D^{\frac{-1}{2}})\\
     & = & \tmop{Tr} ( D^{\frac{-1}{2}} H ( U U^T - \lambda Y Y^T) H D^{\frac{-1}{2}} K_{XW})\\
     & = & \tmop{Tr} ( \gamma K_{XW})\\
     & = & \sum_{i, j} \gamma_{i, j} K_{X_{i, j}}.
   \end{array}  \]
where $\gamma$ is a symmetric matrix and $\gamma = H ( U U^T - \lambda Y Y^T)
H$. By substituting the Gaussian kernel for $K_{X_{i, j}}$, the objective
function becomes
\[ \begin{array}{l}
     \min\\
     W
   \end{array} - \sum_{i, j} \gamma_{i, j} e^{- \frac{\tmop{Tr} [ W^T A_{i, j}
   W]}{2 \sigma^2}} \begin{array}{llll}
     &  & s.t & W^T W = I.
   \end{array} \]
\label{Cost_Derivation}
\end{appendices}

\begin{appendices} 
    \section{Proof for Lemma \ref{basic_lemma} }
    \begin{proof}
        Algorithm \ref{masteralg} sets the smallest $q$ eigenvectors of $\Phi ( W_k)$ as
        $W_{k + 1}$. Since a fixed point $W^{\ast}$ is reached when $W_k = W_{k + 1}$,
        therefore $W^{\ast}$ consists of the smallest eigenvectors of $\Phi (
        W^{\ast})$ and $\Lambda^{\ast}$ corresponds with a diagonal matrix of eigenvavlues.
        Since the eigenvectors of $\Phi ( W^{\ast})$ are orthonormal
        , $W^{\ast^T} W^{\ast} = I$ is also satisfied.
    \end{proof}
    \label{proof_of_lemma_2}
\end{appendices}

\begin{appendices}
\section{Proof for Lemma \ref{lemma:eig} }
\begin{proof}

    Using Equation (\ref{eq:main_cost_function}) as the objective function, the corresponding Lagrangian and its gradient is written as 
    
    \begin{align}
      \mathcal{L} ( W, \Lambda) &= - \sum_{i, j} \gamma_{i, j} e^{- \frac{\tmop{Tr}
      (W^T A_{i, j} W)}{2 \sigma^2}} - \frac{1}{2} \tmop{Tr} ( \Lambda ( W^T W -
      I)) \label{eq:lagrangian},
    \end{align}
   and 
   \begin{equation} \label{eq:gradient_of_lagrangian}
    \nabla_W \mathcal{L} ( W, \Lambda) = \sum_{i, j} \frac{\gamma_{i,
      j}}{\sigma^2} e^{- \frac{\tmop{Tr} (W^T A_{i, j} W)}{2 \sigma^2}} A_{i, j} W
      - W \Lambda. 
    \end{equation}

  By setting the gradient of the Lagrangian to zero, and using the definition of $\Phi(W)$ from Equation (\ref{eq:phi_equation}), Equation (\ref{eq:gradient_of_lagrangian}) can be written as
    \begin{equation}\label{eig_relation} 
     \Phi(W) W = W \Lambda. 
    \end{equation} 
    The gradient with respect to $\Lambda$ is
    
    \begin{align}
        \nabla_\Lambda \mathcal{L} ( W, \Lambda) =  W^T W - I.
    \end{align}
    
    Setting this gradient of the Lagrangian also to zero, condition (\ref{eq:1st_lambda}) is equivalent to
     
    \begin{equation} 
        W^T W = I. \label{prof:orth}
    \end{equation}
     
     By Lemma \ref{basic_lemma}, a fixed point $W^*$ and its corresponding $\Lambda^*$ satisfy (\ref{eig_relation}) and (\ref{prof:orth}), and the lemma follows.
\end{proof}
\label{proof_of_lemma_3}
\end{appendices}

\begin{appendices}
\section{Proof for Lemma \ref{lemma:2nd_order_lemma} }

The proof for Lemma \ref{lemma:2nd_order_lemma} relies on the following three sublemmas. The first two sublemmas demonstrate how the 2nd order conditions can be rewritten into a simpler form.  
With the simpler form, the third lemma demonstrates how the 2nd order conditions of a local minimum are satisfied given a large enough $\sigma$.

\begin{sublemma}\label{lemma:directional_derivate}

Let the directional derivative in the direction of $Z$ be defined as 
\begin{equation}
  \mathcal{D} f (W) [Z] \assign \begin{array}{l}
    \lim\\
    t \rightarrow 0
  \end{array} \frac{f (W + tZ) - f (W)}{t}.
\end{equation}

Then the 2nd order condition of Lemma \ref{lemma:2nd_order_lemma} can be written as 

\begin{equation} \label{cond:large_inequality}
\tmop{Tr} ( Z^T \mathcal{D} \nabla \mathcal{L} [ Z]) = \left\{ \sum_{i, j}
   \frac{\gamma_{i, j}}{\sigma^2} e^{- \frac{\tmop{Tr} ( ( W^{*^T} A_{i, j} W^*
   \nobracket)}{2 \sigma^2}}  \left[ \tmop{Tr} ( Z^T A_{i, j} Z) -
   \frac{1}{\sigma^2} \tmop{Tr} ( Z^T A_{i, j} W^*)^2 \right] \right\} -
   \tmop{Tr} ( Z^T Z \Lambda^*),
\end{equation}

for all $Z$ such that 
\begin{equation} \label{cond:Zcondition}
Z^T W^* + W^{*^T} Z = 0.
\end{equation}

\end{sublemma}

\begin{proof}
Observe first that 
\begin{equation}
    \nabla^2_{W^* W^*} \mathcal{L} ( W^*, \Lambda^*) Z = 
    \mathcal{D} \nabla \mathcal{L} [ Z],
\end{equation}

where the directional derivative of the gradient $\mathcal{D} \nabla \mathcal{L} [ Z]$
is given by 

\[ \mathcal{D} \nabla \mathcal{L} [ Z] = \begin{array}{l}
     \lim\\
     t \rightarrow 0
   \end{array} \frac{\partial}{\partial t} \sum_{i, j} \frac{\gamma_{i,
   j}}{\sigma^2} e^{- \frac{\tmop{Tr} ( ( W^* + t Z)^T A_{i, j}  ( W^* + t Z))}{2
   \sigma^2}} A_{i, j}  ( W^* + t Z) - ( W^* + t Z) \Lambda. \]

This can be written as
\[ \mathcal{D} \nabla \mathcal{L} [ Z] = T_1 + T_2 - T_3, \]
where
\begin{align}
    T_1 &= \begin{array}{l}
        \lim\\ t \rightarrow 0
        \end{array} \frac{\partial}{\partial t} \sum_{i, j} \frac{\gamma_{i,
        j}}{\sigma^2} e^{- \frac{\tmop{Tr} ( ( W^* + t Z)^T A_{i, j}  ( W^* + t Z))}{2
        \sigma^2}} A_{i, j} W^*\\
    &= \begin{array}{l}
        \lim\\
        t \rightarrow 0
        \end{array} \frac{\partial}{\partial t} \sum_{i, j} \frac{\gamma_{i,
        j}}{\sigma^2} e^{- \frac{\tmop{Tr} ( ( W^{*^T} A_{i, j} W^* + t Z^T A_{i, j} W^* +
        t W^{*^T} A_{i, j} Z + t^2 Z^T A_{i, j} Z \nobracket)}{2 \sigma^2}} A_{i, j} W^*  \\
    &= - \sum_{i, j} \frac{\gamma_{i, j}}{2 \sigma^4} e^{- \frac{\tmop{Tr} (
        ( W^{*^T} A_{i, j} W^* \nobracket)}{2 \sigma^2}} \tmop{Tr} ( Z^T A_{i, j} W^* + W^{*^T}
        A_{i, j} Z) A_{i, j} W^* \\
    &= - \sum_{i, j} \frac{\gamma_{i, j}}{\sigma^4} e^{- \frac{\tmop{Tr} (
   ( W^{*^T} A_{i, j} W^* \nobracket)}{2 \sigma^2}} \tmop{Tr} ( Z^T A_{i, j} W^*)
   A_{i, j} W^* &    \text{as $A_{i, j}=A_{i,j}^T$},\\
    T_2 &= \begin{array}{l}
     \lim\\
     t \rightarrow 0
   \end{array} \frac{\partial}{\partial t} \sum_{i, j} \frac{\gamma_{i,
   j}}{\sigma^2} t e^{- \frac{\tmop{Tr} ( ( W^* + t Z)^T A_{i, j}  ( W^* + t
   Z))}{2 \sigma^2}} A_{i, j} Z\\
   &= \sum_{i, j} \frac{\gamma_{i, j}}{\sigma^2} e^{- \frac{\tmop{Tr} (W^{*^T}
   A_{i, j} W^*)}{2 \sigma^2}} A_{i, j} Z, \\
    T_3 &= \begin{array}{l} \lim\\ t \rightarrow 0 \end{array} \frac{\partial}{\partial t}  ( W^* + t Z) \Lambda\\
    &= Z \Lambda.
\end{align}

Hence, putting all three terms together yields

\begin{equation}\label{eq:directional_gradient}
\mathcal{D} \nabla \mathcal{L} [ Z] = \left\{ \sum_{i, j} \frac{\gamma_{i,
   j}}{\sigma^2} e^{- \frac{\tmop{Tr} ( ( W^{*^T} A_{i, j} W^* \nobracket)}{2
   \sigma^2}}  \left[ A_{i, j} Z - \frac{1}{\sigma^2} \tmop{Tr} ( Z^T A_{i, j}
   W^*) A_{i, j} W^* \right] \right\} - Z \Lambda.
\end{equation}

Hence,

\begin{equation}
    \tmop{Tr} (Z^T \nabla^2_{W^* W^*} \mathcal{L} ( W^*, \Lambda^*) Z) = 
    \tmop{Tr} ( Z^T \mathcal{D} \nabla \mathcal{L} [ Z]) ,
\end{equation}

\begin{equation} \label{trace_2nd_condition}
   = \left\{ \sum_{i, j}
   \frac{\gamma_{i, j}}{\sigma^2} e^{- \frac{\tmop{Tr} ( ( W^{*^T} A_{i, j} W^*
   \nobracket)}{2 \sigma^2}}  \left[ \tmop{Tr} ( Z^T A_{i, j} Z) -
   \frac{1}{\sigma^2} \tmop{Tr} ( Z^T A_{i, j} W^*)^2 \right] \right\} -
   \tmop{Tr} ( Z^T Z \Lambda_W).   
\end{equation}

Next, let $Z$ be such that $Z\neq 0$ and $\nabla h (W^{\ast})^T Z = 0$, where
\begin{equation}
h ( W^*) = W^{*^T} W^* - I.
\end{equation}

Therefore, the constraint condition can be written on $Z$ in (\ref{eq:2nd_W}) can be written as

\begin{align}
\begin{split}
\nabla h ( W^*)^T Z &= \begin{array}{l} \lim\\
     t \rightarrow 0
   \end{array} \frac{\partial}{\partial t}  \frac{( W^* + t Z)^T ( W^* + t Z) -
   W^{*^T} W^*}{t}\\
    &= Z^T W^* + W^{*^T} Z = 0. \label{Zconclusion}
\end{split}
\end{align}

Using Equations (\ref{trace_2nd_condition}) and (\ref{Zconclusion})  lemma \ref{lemma:directional_derivate} follows.
\end{proof}

Recall from Lemma \ref{basic_lemma} that $W^*$ consists
of the $q$ eigenvectors \ of $\Phi (W^*)$ with the smallest eigenvalues. We define $\bar{W^*} \in \mathcal{R}^{d \times d - q}$ as all other eigenvectors of $\Phi(W^*)$. \ Because $Z$ has the same dimension as $W^*$, each
column of $Z$ resides in the space of $\mathcal{R}^d$. Since the eigenvectors
of $\Phi (W^*)$ span $\mathcal{R}^d$, each column of $Z$ can be represented
as a linear combination of the eigenvectors of $\Phi (W^*)$. In other words, each
column $z_i$ can therefore be written as $z_i = W^*P^{(i)}_W + \bar{W^*}
P^{(i)}_{\bar{W^*}}$, where $P^{(i)}_{W^*} \in \mathcal{R}^{q \times 1}$ and \
$P^{(i)}_{\bar{W^*}} \in \mathcal{R}^{d - q \times 1}$ \ represents the
coordinates for the two sets of eigenvectors. Using the same notation, we also define $\Lambda^* \in
\mathcal{R}^{q \times q}$ as the eigenvalues corresponding to $W^*$ and
$\bar{\Lambda^*} \in \mathcal{R}^{d - q \times d - q}$ as the eigenvalues
corresponding to $\bar{W^*}$. The entire matrix $Z$ can therefore be
represented as 
\begin{equation} \label{def:Z}
  Z = \bar{W^*} P_{\bar{W^*}} + W^*P_{W^*}.
\end{equation}

Furthermore, it can be easily shown that $P_{W^*}$ is a skew symmetric matrix, or $- P_{W^*} = P_{W^*}^T$. By  setting $Z$ from Equation (\ref{cond:Zcondition}) into  (\ref{def:Z}), the constraint can be rewritten as

\begin{align}
[ P_{\bar{W^*}}^T \bar{W^*}^T + P_W^{*^T} W^{*^T}] W^* + W^{*^T} [ \bar{W^*} P_{\bar{W^*}} + W^* P_{W^*}] &= 0.
\end{align}

Simplifying the equation yields the relationship

\begin{equation}
P_W^{*^T} + P_{W^*} = 0.
\end{equation}

Using these definitions, we define the following sublemma.

\begin{sublemma}\label{lemma:To_final_form}
Given a fixed point $W^*$ and a $Z$ satisfying condition (\ref{cond:Zcondition}), the condition 
$\tmop{Tr} ( Z^T \mathcal{D} \nabla \mathcal{L} [ Z]) \geq 0$ is equivalent to 
\begin{equation} \label{eq:inequality}
   \tmop{Tr} ( P_{\bar{W^*}}^T \bar{\Lambda^*} P_{\bar{W^*}}) - \tmop{Tr} (
   P_{\bar{W^*}} \Lambda^* P_{\bar{W^*}}^T) \geq C_2,
\end{equation}
where 
\begin{equation} \label{def:C2}
   C_2 = \sum_{i, j} \frac{\gamma_{i, j}}{\sigma^4} e^{-
   \frac{\tmop{Tr} ( ( W^{*^T} A_{i, j} W^* \nobracket)}{2 \sigma^2}} \tmop{Tr} (
   Z^T A_{i, j} W^*)^2,
\end{equation}

$P_{W^*}, P_{\bar{W^*}}$ are given by Equation (\ref{def:Z}), and $\Lambda^*,\bar{\Lambda^*}$ are the diagonal matrices containing the bottom and top eigenvalues of $\Phi(W^*)$ respectively.

\end{sublemma}

\begin{proof}
By condition  (\ref{cond:large_inequality}), 
\begin{align}
    \tmop{Tr} ( Z^T \mathcal{D} \nabla \mathcal{L} [ Z]) &= C_1 - C_2 + C_3, \label{eq:41} 
\end{align}
where
\begin{align*}
    C_1  &= \tmop{Tr} \left( Z^T \sum_{i, j} \frac{\gamma_{i,
   j}}{\sigma^2} e^{- \frac{\tmop{Tr} ( ( W^{*^T} A_{i, j} W^* \nobracket)}{2
   \sigma^2}} A_{i, j} Z \right), \\
   C_2 &= \sum_{i, j} \frac{\gamma_{i, j}}{\sigma^4} e^{-
   \frac{\tmop{Tr} ( ( W^{*^T} A_{i, j} W^* \nobracket)}{2 \sigma^2}} \tmop{Tr} (
   Z^T A_{i, j} W^*)^2  \label{eq:43},\\
   C_3 &= - \tmop{Tr} ( Z^T Z \Lambda^*).
\end{align*}

$C_1$ can be written as 

\begin{align*}
 C_1 &= \tmop{Tr} \left( Z^T \sum_{i, j} \frac{\gamma_{i, j}}{\sigma^2} e^{-
   \frac{\tmop{Tr} ( ( W^{*^T} A_{i, j} W^* \nobracket)}{2 \sigma^2}} A_{i, j} Z
   \right) &\\
 &=\tmop{Tr} ( Z^T \Phi ( W^*) [ \bar{W^*} P_{\bar{W^*}} + W^* P_{W^*}]) &\\
 &=\tmop{Tr} ( Z^T [ \Phi ( W^*)  \bar{W^*} P_{\bar{W^*}} + \Phi ( W^*) W^* P_{W^*}])  & \\
 &=\tmop{Tr} ( Z^T [ \bar{W^*}  \bar{\Lambda} P_{\bar{W^*}} + W^* \Lambda P_{W^*}]) 
    & \text{By definition of eigenvalues.} \\
 &= \tmop{Tr} ( [ P_{\bar{W^*}}^T \bar{W^*} ^T + P_W^{*^T} W^{*^T} ] [ \bar{W^*} 
   \bar{\Lambda} P_{\bar{W^*}} + W^* \Lambda P_{W^*}]) & \text{Substitute for $Z$} \\
 &=\tmop{Tr} ( P_{\bar{W^*}}^T  \bar{\Lambda} P_{\bar{W^*}} ) + \tmop{Tr} ( P_{W^*}^T
   \Lambda P_W ) & \text{Given $W^{*^T}W^*=I$, $\bar{W^*}^TW^* =0 $}.
\end{align*}

Similarly
\begin{align*}
    C_3 &= - \tmop{Tr} ( Z^T Z \Lambda) \\
    &= - \tmop{Tr} ( [ P_{\bar{W^*}}^T \bar{W^*} ^T + P_{W^*}^T W^{*^T}] 
        [ \bar{W^*} P_{\bar{W^*}} + W^* P_{W^*}] \Lambda) \\
    &= - \tmop{Tr} ( [ P_{\bar{W^*}}^T P_{\bar{W^*}} + P_{W^*}^T P_{W^*}] \Lambda)\\
    &= - \tmop{Tr} ( P_{\bar{W^*}}^T P_{\bar{W^*}} \Lambda) - 
        \tmop{Tr} ( P_{W^*}^T P_{W^*} \Lambda).
\end{align*}

Because $P_{W^*}$ is a square skew symmetric matrix, the diagonal elements of $P_{W^*}
P_{W^*}^T$ is the same as the diagonal of $P_{W^*} P_{W^*}^T$. From this observation, we
conclude that $\tmop{Tr} ( P_{W^*} P_{W^*}^T \Lambda) = \tmop{Tr} ( P_{W^*}^T P_{W^*}
\Lambda)$. Hence,

\begin{align*}
C_3 = - \tmop{Tr} ( P_{\bar{W^*}} \Lambda P_{\bar{W^*}}^T ) -
   \tmop{Tr} ( P_{W^*}^T \Lambda P_{W^*} ).
\end{align*}

Putting all 3 parts together yields

\begin{align}
\begin{split}
    \tmop{Tr} ( Z^T \mathcal{D} \nabla \mathcal{L} [ Z]) &= \tmop{Tr} (
   P_{\bar{W^*}}^T  \bar{\Lambda} P_{\bar{W^*}} ) + \tmop{Tr} ( P_{W^*}^T \Lambda P_{W^*}
   ) - C_2 - \tmop{Tr} ( P_{\bar{W^*}} \Lambda P_{\bar{W^*}}^T ) -
   \tmop{Tr} ( P_{W^*}^T \Lambda P_{W^*} ) \\
   &=\tmop{Tr} ( P_{\bar{W^*}}^T \bar{\Lambda} P_{\bar{W^*}}) - \tmop{Tr} (
   P_{\bar{W^*}} \Lambda P_{\bar{W^*}}^T) - C_2 .
\end{split}
\end{align}

The 2nd order condition (\ref{eq:2nd_W}) is, therefore, satisfied, when   
\begin{equation} \label{eq:58}
   \tmop{Tr} ( P_{\bar{W^*}}^T \bar{\Lambda} P_{\bar{W^*}}) - \tmop{Tr} (
   P_{\bar{W^*}} \Lambda P_{\bar{W^*}}^T) \geq C_2.
\end{equation}

\end{proof}

\begin{sublemma} \label{large_enough_sigma}
    Given $W^*$,$\bar{W^*}$,$\bar{\Lambda^*}$, and $\Lambda^*$ 
    as defined in Equation (\ref{def:Z}), 
   if the corresponding smallest eigenvalue of $\bar{\Lambda^*}$ 
   is larger than the largest eigenvalue of $\Lambda^*$, then given a large enough  $\sigma$ the 
   condition (\ref{eq:2nd_W}) of Lemma \ref{lemma:2nd_order} is satisfied.
\end{sublemma}

\begin{proof}
To proof sublemma (\ref{large_enough_sigma}), we provide bounds on each of the terms in \eqref{eq:58}. Starting with $C_2$ defined at (\ref{def:C2}). It has a trace term, $(\tmop{Tr}(Z^TA_{ij}W^*))^2$ that can be rewritten as 

\begin{equation}\label{eq:60}
  (\tmop{Tr}(A_{ij}W^*Z^T))^2 = (\tmop{Tr}(A_{ij}W^*P^T_{W^*}W^{*^T}+ A_{ij}W^*P^T_{\bar{W^*}}\bar{W^*}^T))^2.
\end{equation}

Since $A_{ij}$ is symmetric and $W^*P^T_{W^*}W^{*^T}$ is skew-symmetric, then
Tr$(A_{ij}W^*P^T_{W^*}W^{*^T})=0$. Hence 

\begin{align}
   (\tmop{Tr}(Z^TA_{ij}W^*))^2 & =  (\tmop{Tr}(A_{ij}W^*Z^T))^2 
    = (\tmop{Tr}(A_{ij}W^*P^T_{\bar{W^*}}\bar{W^*}^T))^2 \\
    &\leq \tmop{Tr}(A^T_{i,j}A_{ij})\tmop{Tr}(P^T_{\bar{W^*}}P_{\bar{W^*}}) 
\end{align}

where the last inequality follows from Cauchy-Schwartz inequality and that fact that $W^{*^T}W^*=I$ and $\bar{W^*}^T\bar{W^*}=I$.
Thus, $C_2$ in \eqref{eq:43} is bounded by
\begin{equation}\label{eq:c2bound}
   C_2 \leq  \sum_{i, j} \frac{|\gamma_{i, j}|}{\sigma^4} e^{-
   \frac{\tmop{Tr} ( ( W^{*^T} A_{i, j} W^* \nobracket)}{2 \sigma^2}} \tmop{Tr} (
    A^T_{i,j}A_{ij})\tmop{Tr}(P^T_{\bar{W^*}}P_{\bar{W^*}})
\end{equation}
Similarly, the remaining terms in \eqref{eq:41} can be bounded by
\begin{equation}
C_1= \tmop{Tr} ( P_{\bar{W^*}}^T \bar{\Lambda^*} P_{\bar{W^*}}) \geq \min_i
    ( \bar{\Lambda^*}_i) \tmop{Tr} ( P_{\bar{W^*}} P_{\bar{W^*}}^T) 
\end{equation}
\begin{equation}
 C_3 = -\tmop{Tr} ( P_{\bar{W^*}} \Lambda^* P^T_{\bar{W^*}}) \geq -
     \max_i ( \Lambda_i^*) \tmop{Tr} ( P_{\bar{W^*}}^T P_{\bar{W^*}}).
\end{equation}
     
Using the bounds for each term, the Equation (\ref{eq:58}) can be rewritten as
\begin{align}
    [\min_i ( \bar{\Lambda^*}_i) -\max_j(\Lambda_j^*)] \tmop{Tr}(P^T_{\bar{W^*}}P_{\bar{W^*}}) &\geq 
    \sum_{i, j} \frac{|\gamma_{i, j}|}{\sigma^4} e^{-
   \frac{\tmop{Tr} ( ( W^{*^T} A_{i, j} W^* \nobracket)}{2 \sigma^2}}
   \tmop{Tr} (A^T_{i,j}A_{ij})
   \tmop{Tr}(P^T_{\bar{W^*}}P_{\bar{W^*}})\\
   [\min_i ( \bar{\Lambda^*}_i) -\max_j(\Lambda_j^*)] &\geq 
    \sum_{i, j} \frac{|\gamma_{i, j}|}{\sigma^4} e^{-
   \frac{\tmop{Tr} ( ( W^{*^T} A_{i, j} W^* \nobracket)}{2 \sigma^2}} \tmop{Tr} (A^T_{i,j}A_{ij})\label{eq:conclusion}
\end{align}

It should be noted that $\Lambda^*$ is a function of $\frac{1}{\sigma^2}$. This relationship could be removed by multiplying both sides of the inequality by $\sigma^*$ to yield
\begin{equation}
   \sigma^2 [\min_i ( \bar{\Lambda^*}_i) -\max_j(\Lambda_j^*)] \geq 
    \sum_{i, j} \frac{|\gamma_{i, j}|}{\sigma^2} e^{-
   \frac{\tmop{Tr} ( ( W^{*^T} A_{i, j} W^* \nobracket)}{2 \sigma^2}}\tmop{Tr} (A^T_{i,j}A_{ij}). \label{ineq:large_sigma}
\end{equation} 

Since $\sigma^2$ is always a positive value, as long as all the eigenvalues from $\bar{\Lambda^*}$ is larger than all the eigenvalues from $\Lambda^*$, the left hand side of the equation will always be greater than 0. As $\sigma \rightarrow \infty$, the right hand side approaches 0, and the  condition (\ref{eq:2nd_W}) of Lemma \ref{lemma:2nd_order} is satisfied. 
\end{proof} \label{proof_of_lemma_4}

As a side note, the eigen gap between min($\bar{\Lambda^*}$) and max($\Lambda^*$) controls the range of potential $\sigma$ values —i.e. the larger the eigen gap the easier for $\sigma$ to satisfy \eqref{ineq:large_sigma}. Therefore, the ideal cutoff point should have a large eigen gap.

\end{appendices}


\newpage
\begin{appendices} \label{Appendix:Convergence_plot}
\section{Convergence Plot from Experiments}

Figure \ref{fig:convergence} summarizes the convergence activity of various experiments. For each experiment, the top figure 
provides the magnitude of the objective function. It can be seen 
that the values converges towards a fixed point. The middle plot provide updates of the gradient of the Lagrangian. It can be seen that the gradient converges towards 0. The bottom plot shows the changes in $W$ during each iteration. The change in $W$ converge towards 0.

\begin{figure}[ht]
  \centering
  \includegraphics[width=14cm,height=14cm]{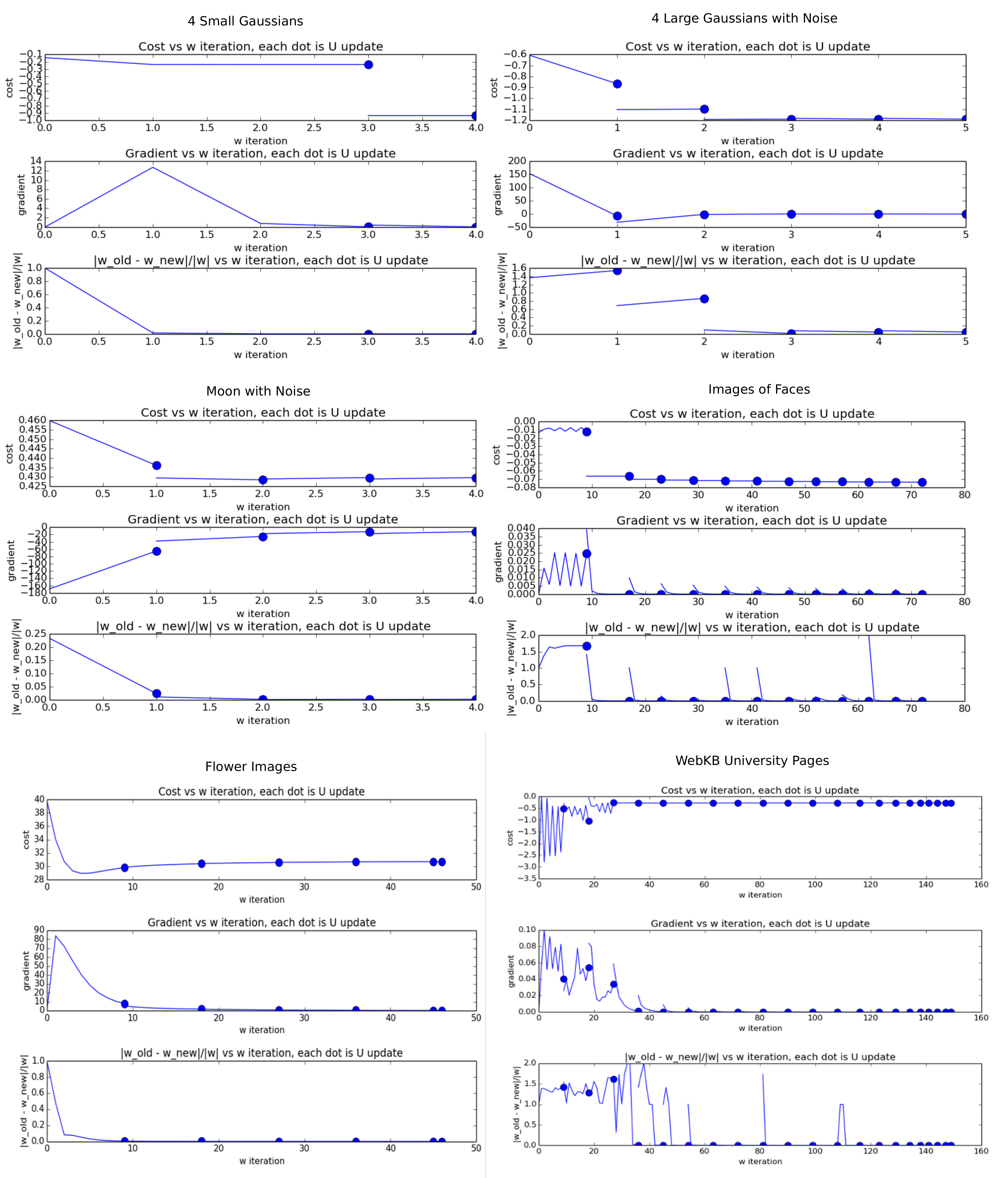}
  \caption{Convergence Results from the Experiments.}
  \label{fig:convergence}
\end{figure}
\end{appendices}
\newpage

\begin{appendices}
\section{Proof of Convergence}

The convergence property of ISM has been analyzed and yields the following theorem.

\begin{theorem}
  A sequence $\{W_k\}_{k\in \mathbb{N}}$ generated by Algorithm~\ref{masteralg} contains a convergent subsequence. 
 \end{theorem}

\begin{proof}

According to Bolzano-Weierstrass theorem, if we can show that the sequences generated from
the 1st order relaxation is bounded, it has a convergent subsequence. If we
study the Equation $\Phi(W)$ more closely, the key driver of the sequence of $W_k$
is the matrix $\Phi$, therefore, if we can show that if this matrix is
bounded, the sequence itself is also bounded. We look inside the construction
of the matrix itself.

\[ \Phi_{n + 1} = \left[ \sum_{i, j} \frac{\gamma_{i, j}}{\sigma^2} e^{-
   \frac{\tmop{Tr} (W_n^T A_{i, j} W_n)}{2 \sigma^2}} A_{i, j} \right] \]

From this equation, start with the matrix $A_{i, j} = ( x_i - x_j) ( x_i -
x_j)^T$. Since $x_i, x_j$ are data points that are always centered and scaled
to a variance of 1, the size of this matrix is always constrained. It also
implies that $A_{i, j}$ is a PSD matrix. From this, the exponential term is always limited between
the value of 0 and 1. The value of $\sigma$ is a constant given from the
initialization stage. Lastly, we have the $\gamma_{i, j}$ term. Since $\gamma
= D^{- 1 / 2} H ( U U^T - \lambda Y Y^T) H D^{- 1 / 2}$. The degree matrix
came from the exponential kernel. Since the kernels are bounded, $D$ is also
bounded. The centering matrix $H$ and the previous clustering result $Y$ can
be considered as bounded constants. Since the spectral embedding $U$ is a
orthonormal matrix, it is always bounded. From this, given that the components
of $\Phi$ is bounded, the infinity norm of the $\Phi$ is always bounded. The
eigenvalue matrix of $\Lambda$ is therefore also bounded. Using the
Bolzano-Weierstrass Theorem, the sequence contains a convergent sub-sequence. Given that $\Phi$ is a continuous function of $W$, by continuity, $W$ also has a convergent sub-sequence.
\end{proof}
\end{appendices}

\begin{appendices}
\section{Proof for the initialization}
Although the proof was originally shown through the usage of the 2nd order Taylor Approximation. A simpler approach was later discovered to arrive to the same formulation faster. We first note that Taylor's Expansion around 0 of an exponential is 
\[e^x=1+x+\frac{x^2}{2!}+...  .\]

Given the objective Lagrangian in eq \eqref{eq:lag}, \ we simplify the Lagrangian by using the Taylor approximation only on the problematic exponential term. The approximation is expanded up to the 1st order centering around 0 to yield 
\[ \mathcal{L} \approx - \sum_{i, j} \gamma_{i, j}  \left( 1 - \frac{\tmop{Tr}
   (W^T A_{i, j} W)}{2 \sigma^2} \right) + \frac{1}{2} \tmop{Tr} (\Lambda (I -
   W^T W)) .\]
By taking the derivative of the approximated Lagrangian and setting the derivative to zero, an eigenvalue/eigenvector relationship emerges as
\[ \Phi W =  \left[ \sum_{i, j} \frac{\gamma_{i, j}}{\sigma^2} A_{i,
   j} \right] W_0 = W_0 \Lambda .\]
   
From this, we see that $\Phi_0$ is no longer a function of $W$. Using this $\Phi_0$ we can then calculate a closed form solution for $W_0$
\end{appendices}

\begin{appendices}
\section{Proof for the computational complexity}

For ISM, DG and SM, the bottleneck resides in the computation of the gradient.
\[ f ( W) = \sum_{i, j} \gamma_{i, j} e^{- \frac{\tmop{Tr} ( W^T A_{i, j}
   W)}{2 \sigma^2}} \]
\[ \frac{\partial f}{\partial W} = \left[ \sum_{i, j}  \frac{\gamma_{i,
   j}}{\sigma^2} e^{- \frac{\tmop{Tr} ( W^T A_{i, j} W)}{2 \sigma^2}} A_{i, j}
   \right] W \]
\[ \frac{\partial f}{\partial W} = \left[ \sum_{i, j}  \frac{\gamma_{i,
   j}}{\sigma^2} e^{- \frac{\tmop{Tr} ( W^T \Delta x_{i, j} \Delta x_{i, j}^T
   W)}{2 \sigma^2}} A_{i, j} \right] W \]
Where $A_{i, j} = \Delta x_{i, j} \Delta x_{i, j}^T$. The variables have the
following dimensions.
\[ \begin{array}{l}
     x_{i, j} \in \mathcal{R}^{d \times 1}\\
     W \in \mathcal{R}^{d \times q}
   \end{array} \]
To compute a new $W$ with DG, we first mulitply $\Delta x_{i, j}^T W$, which
is $O ( d)$. Note that $W$ in DG is always 1 single column. Next, it
multiplies with its own transpose to yied $O ( d \noplus + q^2)$. Then we
compute $A_{i, j}$ to get $O ( d + q^2 + d^2)$. Since this operation needs to
be added $n^2$ times, we get, $O ( n^2 ( d + q^2 + d^2))$. Since $d \gg q$,
this notation reduces down to $O ( n^2 d^2)$. Let $T_1$ be the number of
iterations until convergence, then it becomes $O ( T_1 n^2 d^2)$. Lastly, in
DG, this operation needs to be repeated $q$ times, hence, $O ( T_1 n^2 d^2
q)$.

To compute a new $W$ with SM, we first mulitply $\Delta x_{i, j}^T W$, which
is $O ( d q)$. Next, it multiplies with its own transpose to yied $O ( d q
\noplus + q^2)$. Then we compute $A_{i, j}$ to get $O ( d q + q^2 + d^2)$.
Since this operation needs to be added $n^2$ times, we get, $O ( n^2 ( d q +
q^2 + d^2))$. Since $d \gg q$, this notation reduces down to $O ( n^2 d^2)$.
The SM method requires the computation of the inverse of $d \times d$ matrix.
Since inverses is cubic, it becomes $O ( n^2 d^2 \noplus + d^3)$. Lastly, let
$T_2$ be the number of iterations until convergence, then it becomes $O ( T_2
( n^2 d^2 + d^3))$.

To compute a new $W$ with ISM, we first mulitply $\Delta x_{i, j}^T W$, which
is $O ( d q)$. Next, it multiplies with its own transpose to yied $O ( d q
\noplus + q^2)$. Then we compute $A_{i, j}$ to get $O ( d q + q^2 + d^2)$.
Since this operation needs to be added $n^2$ times, we get, $O ( n^2 ( d q +
q^2 + d^2))$. Since $d \gg q$, this notation reduces down to $O ( n^2 d^2)$.
The ISM method requires the computation of the eigen decomposition of $d
\times d$ matrix. Since inverses is cubic, it becomes $O ( n^2 d^2 \noplus +
d^3)$. Lastly, let $T_3$ be the number of iterations until convergence, then
it becomes $O ( T_3 ( n^2 d^2 + d^3))$.

\end{appendices}

\begin{appendices} 
\section{Measure of Non-linear Relationship by HSIC Versus Correlation}

The figure below demonstrates a visual comparison of HSIC and correlation. It can be seen that HSIC  measures non-linear relationships, while correlation does not. 
\begin{figure}[ht]
  \centering
  \includegraphics[width=14cm,height=3cm]{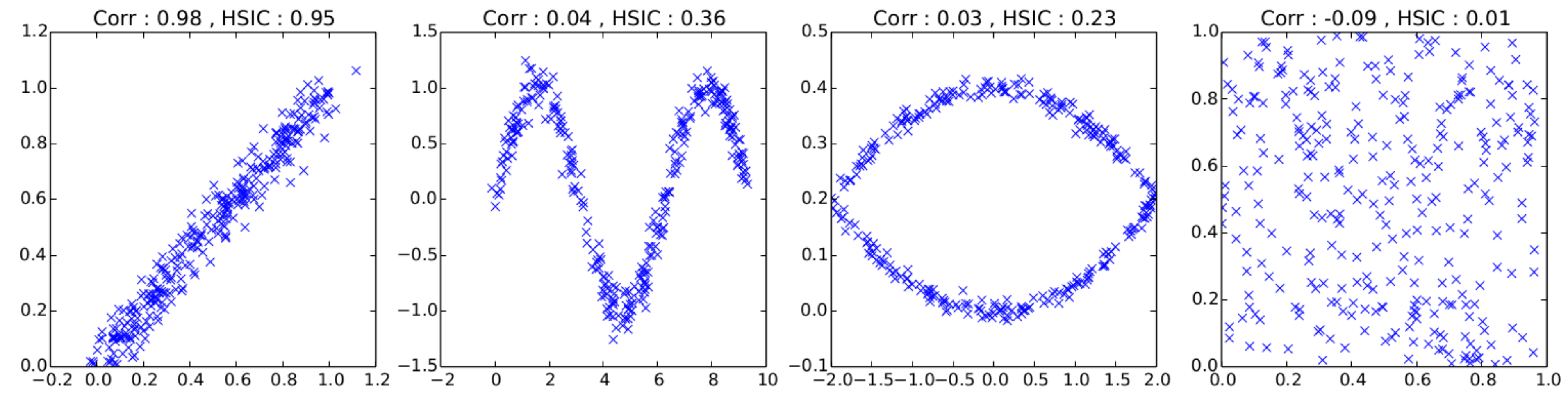}
  \caption{Showing that HSIC captures non-linear information.}
  \label{HSIC_capture_nonlinear}
\end{figure}
\label{App:HSIC}
\end{appendices}

\begin{appendices}
\section{Implementation Details of the Cost function}

The computation of the cost using the formulation below is slow if it is implemented using a loop. 
\begin{equation}
  \begin{array}{ll}
    \min & - \sum_{i, j} \gamma_{i, j} e^{- \frac{\tmop{tr} ( W^T A_{i, j}
    W)}{2 \sigma^2}}\\
    W & \\
    s.t & W^T W = I\\
    & W \in \mathbb{R}^{d \times q}\\
    & A \in \mathbb{R}^{d \times d}\\
    & \gamma_{i, j} \in \mathbb{R}
  \end{array}
\end{equation}

Instead, we use the original formulation to derive a faster way to compute the cost.

Starting with the original cost function as
\[ \tmop{cost} = \tmop{HSIC} ( X W, U) - \lambda \tmop{HSIC} ( X W, Y) \]
\[ \tmop{cost} = \tmop{Tr} ( D^{-1/2}K_{XW} D^{-1/2} H U U^T H) - \lambda \tmop{Tr} (D^{-1/2}K_{XW} D^{-1/2} H Y Y^T H) \]

When optimizing $U$, it is obvious that the 2nd portion does not effect the
optimization. Therefore, $U$ can be solved using the following form.
\[ U = \begin{array}{l}
     \tmop{argmin}\\
     U
   \end{array} \tmop{Tr} ( U^T H D^{- 1 / 2} K_{X W} D^{- 1 / 2} H U) \]

If we are optimization for $W$, using the combination of the rotation property and the combination of the 2 traces, the cost can be written as 
\[ \tmop{cost} = \tmop{Tr} ( [ D^{- 1 / 2} H ( U U^T - \lambda Y Y^T) H D^{-
   1 / 2}] K). \]

In this form, it can be seen that the update of $W$ matrix will only affect
the kernel $K$ and the degree matrix $D$. Therefore, it makes sens to treat
the middle portion as a constant which we refer as $\Psi$.
\[ \tmop{cost} = \tmop{Tr} ( [ D^{- 1 / 2} \Psi D^{- 1 / 2}] K) \]

Given that $[ D^{- 1 / 2} \Psi D^{- 1 / 2}]$ is a symmetric matrix, from this
form, we can convert the trace into an element wise product $\odot$.
\[ \tmop{cost} = \sum_{i,j} ([ D^{- 1 / 2} \Psi D^{- 1 / 2}] \odot K )_{i,j} \]

To further reduction the amount of operation, we let $d$ be a vector of the
diagonal elements of $D^{- 1 / 2}$, hence $d = \tmop{diag} ( D^{- 1 / 2})$,
this equality hold.
\[ D^{- 1 / 2} \Psi D^{- 1 / 2} = [ d d^T] \odot \Psi \]

Therefore, the final cost function can be written in its simplest form as :
\[ \tmop{cost} = \sum_{i,j} \Gamma_{i,j} = \sum_{i,j} (\Psi \odot [ d d^T] \odot K )_{i,j} \]

During update, as $W$ update during each iteration, the matrix $\Psi$ stays as
a constant while $d d^T$ and $K$ update. The benefit of this form minimize the
complexity of the equation, while simplify cost into easily parallelizable
matrix multiplications. The equation also clearly separates the elements into
portions that require an update and portions that does not.
\end{appendices}

\begin{appendices}
\section{Implementation Details of the Derivative}
As it was shown from previous sections, the gradient of our cost function using the Gaussian Kernel is 
\[ \nabla f ( W) = \left[ \frac{1}{\sigma^2} \sum \gamma_{i, j} K_{i, j} A_{i, j} \right] W  - 2W\Lambda\].

If we let $\Psi = \left[  \frac{1}{2\sigma^2}  \gamma_{i, j} K_{i, j}  \right]$ , it can be rewritten as 
\[ \nabla f ( W) = 
\left[ \sum \Psi_{i, j} A_{i, j} \right] W - W\Lambda.\]

From this formulation, the optimal $W$ is equivalent to the eigenvectors of the $\left[ \sum \Psi_{i, j} A_{i, j} \right]$. According to ISM, the $q$ eigenvectors corresponding to the smallest eigenvalues is used for $W$. Since solving $\sum \Psi_{i, j} A_{i, j}$ using a loop is slow, we vectorize the formulation so that

\[ \sum \Psi_{i, j} A_{i, j} = X^T[D_{\Psi} - \Psi]X \], 

where $D_\Psi$ is the degree matrix of $Psi$ such that the diagonal elements are defined as

\[ d_{i,i} = \sum_{j} \Psi_{i,j} \], 

and $X \in \mathbb{R}^{n \times d}$ is the original data.

\end{appendices}

\begin{appendices} 
\section{Hyperparameters Used in Each Experiment}
    \begin{table}[h]
    \centering
        \begin{tabular}{l|ccc}
        & $\sigma$ & $\lambda$ & $q$ \\ \cline{1-4}
        Gauss A & 1 & 0.04 & 1 \\
        Gauss B 200 & 5 & 2 & 3 \\
        Moon 400 & 0.1 & 1 & 3 \\
        Moon+N 200 & 0.2 & 0.1 & 6 \\
        Flower & 2 & 10 & 2 \\
        Face & 3.1 & 1 & 17 \\
        Web KB & 18.7 & 0.057 & 4 \\
        \end{tabular}
    \end{table}
    \label{App:Hyperparameters}
\end{appendices}

\end{document}